\documentclass{article}

\usepackage{amsmath, amsthm}
\usepackage{graphicx}
\usepackage{subcaption}

\newtheorem{lemma}{Lemma}
\newtheorem{theorem}{Theorem}
\newtheorem{proposition}{Proposition}

\usepackage[preprint]{neurips_2026}

\usepackage[utf8]{inputenc} 
\usepackage[T1]{fontenc}    
\usepackage{hyperref}       
\usepackage{url}            
\usepackage{booktabs}       
\usepackage{amsfonts}       
\usepackage{nicefrac}       
\usepackage{microtype}      
\usepackage{xcolor}         

\title{Reducing Diffusion Model Memorization with Higher Order Langevin Dynamics}

\author{%
  Benjamin Sterling \\ 
  Department of Applied Math \& Statistics\\
  Stony Brook University\\
  Stony Brook, NY 11790 \\
  \texttt{benjamin.sterling@stonybrook.edu} \\
  \And
  M\'onica F.~Bugallo \\
  Department of Electrical and Computer Engineering \\
  Stony Brook University \\
  Stony Brook, NY 11790 \\
  \texttt{monica.bugallo@stonybrook.edu} \\
  \AND
  Tom Tirer \\
  Faculty of Engineering \\
  Bar-Ilan University\\
  Ramat-Gan 5290002, Israel \\
  \texttt{tirer.tom@biu.ac.il} \\  
}

\begin{document}

\maketitle

\begin{abstract}

Diffusion/score-based models have emerged as powerful generative models, capable of generating high-quality samples that mimic the training data distribution. However, it has been observed that they are prone to reproducing training samples---known as ``memorization''---potentially violating copyright and privacy. In this paper, we study the effect of Higher-Order Langevin Dynamics (HOLD) on this phenomenon. HOLD diffusion processes introduce auxiliary variables; if the data variable is interpreted as ``position,'' then the auxiliary variables can be interpreted as ``velocity'' and ``acceleration,'' depending on the chosen order of the model. They were originally proposed based on the intuition that they regularize the trajectories of the data variable by implicitly imposing additional dynamical constraints. Our work provides, to our knowledge, the first theoretical characterization of the regularization effect of HOLD. Specifically, we show that in HOLD, the dynamics of the data variable are governed by a low-pass-filtered version of the learned score function, with smoothness increasing with the order of HOLD. We then analyze the optimal empirical score and the possibility of distribution collapse. Together, our results explain the mitigation of memorization as the model order increases. Finally, we present an empirical study on real-world data that supports our theory and highlights this distinct advantage of HOLD over standard diffusion in practice. 


\end{abstract}

\section{Introduction}

First introduced in \citep{diffusion2015} and later refined in \citep{song2019generative,diffusiondenoising, diffusioncts}, Denoising Diffusion Models (DDMs) have demonstrated remarkable success in generating high quality samples from intractable, high-dimensional data distributions. They are based on generating samples from a tractable latent distribution (typically, standard normal) and converting them to samples from the target distribution via iterative denoising with gradually decreasing noise level. 
Trained DDMs have also emerged as powerful generative priors for editing and restoring signals and images \citep{meng2022sdedit,kawar2022denoising,abu2022adir,garber2024image,garber2025zero}.
Recently, however, it has been observed that DDMs are prone to exactly reproducing training samples \citep{carlini2023extracting,somepalli2023diffusion}. This phenomenon, termed ``memorization'', can have severe negative implications, such as violating copyright and privacy.

Memorization of DDMs has recently become a topic of interest both experimentally and theoretically. \citep{yoon2023diffusion} suggested that diffusion models either ``generalize'' or ``memorize'' in a mutually exclusive manner, and that overfitting in this context is not ``benign'' (unlike certain phenomena in linear regression  \citep{benignoverfitting}). 
In fact, a key factor for diffusion memorization is that, if one considers an unconstrained denoiser/score function, the optimum of the empirical mean squared error loss used in training has an analytical form \citep{karras2022elucidating}, which necessarily outputs a sample from the training set as the noise level decreases \citep{yi2023generalization,zhang2024emergence}
(more details in Section~\ref{sec:Background}).
The work of \citep{kadkhodaie2024generalization} argues that diffusion models generalize because of inductive biases; their network architecture encourages shrinkage towards geometry-adaptive harmonic bases that are informed by the data distribution. \citep{zhang2024emergence} performs similar experiments, demonstrating a separation between memorization and generalization regimes depending on the dataset size. They argue that in the generalization regime, the learned distributions are largely independent of model architectures and are instead most heavily influenced by the training data. \citep{gu2025on} empirically studies the relationship between dataset size, the memorization ratio, and different state-of-the-art architectures. 
Other works on diffusion memorization \citep{bonnaire2025why, george2026denoising} include theoretical analyses based on the random features model.
In particular, \citep{bonnaire2025why} empirically and theoretically studies the level of memorization along the optimization.
Their results suggest that one can avoid or mitigate memorization using early stopping within a range of time that depends on the dataset size. However, in practice, especially in the presence of a biased validation dataset, it is hard to gauge when to stop. 
This study motivates the question of whether there are other ways to reduce memorization besides early stopping.

In parallel to the works analyzing memorization in diffusion models, there have been separate methods developed with the goals of shortening and smoothing the diffusion model's path from target and latent distributions through the introduction of higher-order auxiliary variables. Intuitively, if the data variable is interpreted as ``position,'' then the auxiliary variables can be interpreted as ``velocity'' and ``acceleration,'' depending on the chosen order of the model. 
\citep{dockhorn2021score} proposed critically-damped Langevin dynamics (CLD), which introduces a velocity variable, and studies the effect of critical damping on the diffusion dynamics. In this context, critical-damping refers to choosing parameters so that the diffusion forward process's matrix has a single geometric eigenvalue and thus increases the speed of convergence along diffusion time. Inspired by this work, higher-order Langevin dynamics (HOLD) \citep{hold} was introduced to simplify the diffusion process of CLD as an Ornstein–Uhlenbeck process superimposed with a skew symmetric variable coupling. Follow ups of this work have showed that it is possible to critically damp HOLD \citep{sterling,nold}, and that using HOLD helps to defend DDMs against membership inference attacks \citep{sterling2025defendingdiffusionmodelsmembership}.

In this paper, we study the effect of higher-order diffusion modeling on memorization.
First, we close a gap in the literature regarding the regularization effect of these models. These models were originally proposed based on the intuition that they regularize the trajectories of the data variable by implicitly imposing additional dynamical constraints; this regularization effect lacks theoretical justification. We provide, to our knowledge, the first rigorous characterization of the regularization effect of HOLD. Specifically, we utilize the Laplace transform to demonstrate the low-pass filtering effect of the model order on the learned score function. Then, we analyze the optimal empirical score function of HOLD models and present a result on the impossibility of distribution collapse. Altogether, these findings explain the mitigation of memorization as the model order increases.
Finally, we present an empirical study on real-world data that supports our theory. Namely, we measure the generation quality and memorization rate of models of different orders along training for two different datasets:
the CelebA dataset (extending the setup of \citep{bonnaire2025why}) and the CIFAR-10 dataset.
The experiments show that HOLD models offer similar Fr\'{e}chet Inception Distance (FID) results as plain (first-order) diffusion models, while memorization, measured as in previous works, is delayed and attenuated as the model order increases, aligned with our theory.

\section{Background}
\label{sec:Background}

\subsection{Score-based generative modeling and memorization by the optimal empirical score}

Previous works \citep{karras2022elucidating,yi2023generalization} have derived that the regular score matching, which uses an empirical training objective, has a closed form optimal score function 
that is based on a weighted average of the training data and even converges to being dependent solely on the single training sample closest to its input as $t \to 0$. 

Let us present this formally for the Ornstein–Uhlenbeck process, a standard DDM framework, which possesses the following stochastic differential equation (SDE):
\[d\mathbf{x}_t = -\xi \mathbf{x}_t dt + \sqrt{2 \xi L^{-1}}d\mathbf{w}_t,\]
where $L^{-1}$ and $\xi$ are algorithmic parameters and $\mathbf{w}_t$ represents the standard Wiener process.
In this forward process,
given a data sample $\mathbf{x}_0$, a sample $\mathbf{x}_t$ can be drawn according to $\mathbf{x}_t = \exp(-\xi t)\mathbf{x}_0 + \sigma_t \boldsymbol{\epsilon}$, where $\sigma_t^2 = L^{-1}(1 - \exp(-2\xi t))$ and $\boldsymbol{\epsilon} \sim \mathcal{N}(\mathbf{0}, \mathbf{I})$. 

The generation of data samples by  DDMs is based on the fact that there is an associated backward SDE with the same marginal distribution of $\mathbf{x}_t$ \citep{anderson1982reverse,diffusioncts}:
\begin{align}
    d\mathbf{x}_t &= -\xi \left( \mathbf{x}_t + 2L^{-1}\mathbf{s}(\mathbf{x}_t, t)\right) dt + \sqrt{2\xi L^{-1}}d\Bar{\mathbf{w}}_t,
    \label{eq:sde_backward}
\end{align}
where $\mathbf{s}(\cdot, t)=\nabla_{\mathbf{x}}\log p_t(\cdot)$ is the score function of $\mathbf{x}_t$ (with $p_t$ denoting the distribution of $\mathbf{x}_t$). The score function is modeled by $\mathbf{s}_\theta(\cdot, t)$ and learned during training.
For the standard training objective
\begin{align}
    \mathcal{L} &= \mathbb{E}_{t\sim \mathcal{U}(0, 1), \mathbf{x}_0, \boldsymbol{\epsilon} \sim \mathcal{N}(\mathbf{0},\mathbf{I})}|| \boldsymbol{\epsilon} + \sigma_t \mathbf{s}_\theta(\mathbf{x}_t, t) ||^2,
\end{align}
with the empirical distribution $p(\mathbf{x}_0)=\frac{1}{n_{\text{train}}}\sum\nolimits_{k=1}^{n_{\text{train}}} \delta(\mathbf{x}_0-\mathbf{x}_0^{(k)})$,
the optimal (unconstrained) score function is given by $\mathbf{s}_{\text{emp}}(\mathbf{x}, t) = \nabla_{\mathbf{x}}\log p^{\text{emp, OU}}_t(\mathbf{x})$, where
\begin{align*}
    p^{\text{emp, OU}}_t(\mathbf{x}) &= \frac{1}{n_{\text{train}}} \sum_{k=1}^{n_{\text{train}}} (2\pi  \sigma_t^2)^{-h/2} \exp \left(-\frac{1}{2\sigma_t^2}\bigg \| \mathbf{x} - \exp(-\xi t)\mathbf{x}^{(k)}_0 \bigg \|^2 \right).
\end{align*}
That is,
\begin{align*}
    s_{\text{emp}}(\mathbf{x}, t) &= - \frac{1}{\sigma_t^2} \left ( \mathbf{x} - \exp(-\xi t)\frac{\sum_{k=1}^{n_{\mathrm{train}}}\mathcal{N}(\mathbf{x} \mid \exp(-\xi t)\mathbf{x}^{(k)}_0,\sigma_t^2 \mathbf{I})
    \,\, \mathbf{x}^{(k)}_0}
    {\sum_{k=1}^{n_{\mathrm{train}}}\mathcal{N}(\mathbf{x} \mid \exp(-\xi t)\mathbf{x}^{(k)}_0,\sigma_t^2 \mathbf{I})} \right ).
\end{align*}
which follows the description above (i.e., weighted average of the training samples that converges to the nearest training sample as $t \to 0$). 

This conveys that diffusion models generalize due to explicit and implicit inductive biases, while approaching the unconstrained empirical optimum, e.g., due to a small number of training samples \citep{kadkhodaie2024generalization,zhang2024emergence} and/or prolonged training \citep{bonnaire2025why, george2026denoising}, increases memorization.

\subsection{Overview of HOLD}

We start this section by providing an overview of HOLD \citep{hold,nold} following the variable conventions used in \citep{dockhorn2021score}. Suppose a data point is expressed as a vector $\mathbf{x}_0 \in \mathbb{R}^h$. Let $\alpha, L^{-1}, \{\gamma_k\}_{1 \leq k \leq n-1},$ and $\xi$ denote algorithmic parameters, where $n$ denotes the order of the model. In HOLD, we define the initial diffusion variable $\mathbf{u}_0 = \operatorname{vec}(\mathbf{x}_0, \mathbf{v}^{(1)}_0, \mathbf{v}^{(2)}_0 \ldots \mathbf{v}^{(n-1)}_0)$, where $\mathbf{v}^{(1)}_0, \mathbf{v}^{(2)}_0 \ldots \mathbf{v}^{(n-1)}_0 \sim_{\text{iid}} \mathcal{N}(\mathbf{0}, \alpha L^{-1}\mathbf{I})$ are auxiliary variables. Take the following system:
\begin{align*}
    \mathbf{F} &= \sum_{i=1}^{n-1} \gamma_i \left( \mathbf{E}_{i, i+1} - \mathbf{E}_{i+1, i}\right) - \xi \mathbf{E}_{n,n}, \quad
    \mathbf{G} = \sqrt{2\xi L^{-1}} \mathbf{E}_{n,n},
\end{align*}
where $\mathbf{E}_{i,j} \in \mathbb{R}^{n \times n}$ is the matrix of all zeros with a one at index pair $(i,j)$. The forward process evolves according to the SDE:
\begin{align}
    \label{eq:fwdSDE}
    d\mathbf{u}_t &= \mathcal{F}\mathbf{u}_t dt + \mathcal{G}d\mathbf{w}_t,
\end{align}
where $\mathcal{F} = \mathbf{F} \otimes \mathbf{I}_h$, $\mathcal{G} = \mathbf{G} \otimes \mathbf{I}_h$ (with $\otimes$ denoting the Kronecker product), and $\mathbf{w}_t$ represents the standard Wiener process. For example, when $n=3$, the dynamics are governed by:
\begin{align*}
\begin{cases}
    d\mathbf{x}_t &= \gamma_1 \mathbf{v}^{(1)}_t dt,\\
    d\mathbf{v}^{(1)}_t &= \left(-\gamma_1 \mathbf{x}_t + \gamma_2 \mathbf{v}^{(2)}_t \right)dt, \\
    d\mathbf{v}^{(2)}_t &= \left(-\gamma_2 \mathbf{v}^{(1)}_t  - \xi \mathbf{v}^{(2)}_t\right)dt + \sqrt{2\xi L^{-1}} d\mathbf{w}_t.
\end{cases}
\end{align*}
Here, the dynamics of $\mathbf{x}_t$ are modeled by the ``velocity'' variable $\mathbf{v}^{(1)}_t$ and the ``acceleration'' variable $\mathbf{v}^{(2)}_t$.
In the notation of Equation \eqref{eq:fwdSDE}, the $\mathbf{F}$ and $\mathbf{G}$ matrices for $n=3$ become:
\begin{align*}
    \mathbf{F} &= \begin{pmatrix}
        0 & \gamma_1 & 0 \\
        -\gamma_1 & 0 & \gamma_2 \\
        0 & -\gamma_2 & -\xi
    \end{pmatrix}, \quad 
    \mathbf{G} = \begin{pmatrix}
        0 & 0 & 0 \\
        0 & 0 & 0 \\
        0 & 0 & \sqrt{2\xi L^{-1}}
    \end{pmatrix}.
\end{align*}

It can be shown that $\mathbf{u}_t$ that satisfies Equation \eqref{eq:fwdSDE} must be normally distributed (conditioned on $\mathbf{x}_0$), and expressions for the mean and covariance of $\mathbf{u}_t$ must satisfy the following differential equations, which can be derived from the Fokker-Planck equations \citep{sarkka2019applied},
\begin{align*}
    \frac{d \boldsymbol{\mu}_t}{dt} &= \mathcal{F}\boldsymbol{\mu}_t, \quad \frac{d \boldsymbol{\Sigma}_t}{dt} = \mathcal{F}\boldsymbol{\Sigma}_t + \left( \mathcal{F} \boldsymbol{\Sigma}_t\right)^T + \mathcal{G}\mathcal{G}^T.
\end{align*}
It is proven in \citep{nold} that $\mathbf{u}_t$ must also possess the following distribution:
\begin{align*}
    \mathbf{u}_t &\sim \mathcal{N}(\boldsymbol{\mu}_t, \boldsymbol{\Sigma}_t), \quad 
    \boldsymbol{\mu}_t = \exp(\mathcal{F}t)\mathbf{u}_0 ,\\
    \boldsymbol{\Sigma}_t &= L^{-1}\mathbf{I} + \exp(\mathcal{F}t) \left( \boldsymbol{\Sigma}_0 - L^{-1}\mathbf{I}\right) \exp(\mathcal{F}t)^T.
\end{align*}
In the forward diffusion process, 
samples are generated according to $\mathbf{u}_t = \exp(\mathcal{F}t)\mathbf{u}_0 + \mathbf{L}_t \boldsymbol{\epsilon}_{\text{full}}$, where $\boldsymbol{\epsilon}_1, \boldsymbol{\epsilon}_2 \ldots \boldsymbol{\epsilon}_n \sim \mathcal{N}(\mathbf{0}, \mathbf{I}_h), \boldsymbol{\epsilon}_{\text{full}} = \operatorname{vec}(\boldsymbol{\epsilon}_1, \boldsymbol{\epsilon}_2 \ldots \boldsymbol{\epsilon}_n),$ and $ \mathbf{L}_t = \operatorname{cholesky}(\boldsymbol{\Sigma}_t)$.

The computation of the term $\exp(\mathcal{F}t)$ is not obvious, and \citep{hold} uses Putzer's Lemma \citep{putzer} to compute it. However, when the parameters of $\mathcal{F}$ are chosen so that it posses a single geometric eigenvalue $s_*=\sqrt{2n-3}$, as in \citep{nold}, then it may be calculated by a finite Taylor Series, understanding that $\mathcal{F} - s_*\mathbf{I}$ is nilpotent:
\begin{align} \label{eq:expFt}
    \exp(\mathcal{F}t) = \exp(s_* t) \sum_{k=0}^{n-1}\frac{(\mathcal{F} - s_* \mathbf{I})^k t^k}{k!}.
\end{align}
The loss function under this forward SDE is the same score matching objective that \citep{dockhorn2021score, hold} use. Instead of modeling the score of the data variable $\mathbf{x}_t$ or the score of the full $\mathbf{u}_t$, it suffices to model the score of the last auxiliary variable $\mathbf{v}^{(n-1)}_t$ by $\mathbf{s}_\theta(\mathbf{u}_t,t)$, and learn it with the objective
\begin{align}
    \mathcal{L} &= \mathbb{E}_{t\sim \mathcal{U}(0, 1), \mathbf{u}_0, \boldsymbol{\epsilon}_{\text{full}}}\| \boldsymbol{\epsilon}_n + \mathbf{s}_\theta(\mathbf{u}_t, t)\left(\mathbf{L}_t[nh, nh]\right) \|^2.
    \label{eq:lossfcn}
\end{align}

The following ordinary differential equation (ODE) \citep{diffusioncts} shares the same approximate marginal probability densities as the backwards SDE and may be used for sample generation:
\begin{align}
    d\mathbf{u}_t &= \left(\mathcal{F}\mathbf{u}_t - \frac{1}{2} \mathcal{G}\mathcal{G}^T \mathbf{S}_\theta(\mathbf{u}_t, t)\right) dt,
    \label{eq:ode}
\end{align}
where $\mathbf{S}_\theta(\mathbf{u}_t, t) = \operatorname{vec}(\mathbf{0}_{(n-1)h}, \mathbf{s}_\theta(\mathbf{u}_t, t))$. It is further proven in \citep{nold} that the system is critically damped for $n \geq 2$, meaning $\mathbf{F}$ has a single geometric eigenvalue, if and only if
\begin{align*}
    \gamma_{n-i} &= \sqrt{2n-3}\sqrt{\frac{n^2-i^2}{4i^2-1}}, \quad \xi = n\sqrt{2n-3},
\end{align*}
assuming a scaling choice of $\gamma_1=1$. For further convenience in this paper, define $\Bar{\gamma} = \prod_{i=1}^{n-1} \gamma_i$.

\section{Theoretical Analysis}
\label{sec:theoreticalanalysis}

This section presents our theoretical analysis of the regularization introduced by the HOLD model and its ability to mitigate memorization.

\subsection{In HOLD the score is being low-pass filtered}

This subsection argues that HOLD introduces a low-pass filtering effect on the score during the generation procedure. While neither the reverse SDE nor reverse ODE have closed form solutions, we study the reverse ODE given by Equation \eqref{eq:ode} and analyze the relationship between the score function and generated samples from a signals and systems perspective \citep{oppenheim1997signals}. 

We start with Lemma \ref{lem:qspolynomial} below, which assists the proof of our key Theorem \ref{thm:residuethm}.
Essentially, it is based on taking the Laplace transform of both sides of Equation \eqref{eq:ode}, and identifying recursive patterns. 

\begin{lemma}\label{lem:qspolynomial}
Consider the ODE \eqref{eq:ode}.
If the Laplace transform of $\mathbf{x}_t$ is $\tilde{\mathbf{x}}(s)$, the Laplace transform of $\mathbf{s}_{\theta}(\mathbf{u}_t, t)$ is $\tilde{\mathbf{s}}(s)$ (accounting for $t$ from both arguments), and $P(s)$ is the characteristic polynomial of $\mathbf{F}$, then
\begin{align*}
    \tilde{\mathbf{x}}(s) &= - \frac{\Bar{\gamma} \xi L^{-1}\tilde{\mathbf{s}}(s)}{P(s)} + \tilde{\mathbf{x}}^{\mathrm{natural}}(s),
\end{align*}
where $\tilde{\mathbf{x}}^{\mathrm{natural}}(s)$ arises from nonzero $\mathbf{u}_0$ and is independent of the score.
Its full expression is available in the appendix.
\end{lemma}

See Appendix \ref{pf:qspolynomial} for the proof. Now that we have this result, we may proceed to the following theorem.

\begin{theorem} \label{thm:residuethm}
    Let $h_t^{(n)} = -\Bar{\gamma}n\sqrt{2n-3}L^{-1}t^{n-1}\exp(-t\sqrt{2n-3})$ for $n \geq 2$. For model order $n \geq 2$, the solution to \eqref{eq:ode} for the data variable $\mathbf{x}_t$ is:
    \[\mathbf{x}_t = h_t^{(n)} * \mathbf{s}_\theta(\mathbf{u}_t, t) + \mathbf{x}^{\mathrm{natural}}_t,\]
    where `$*$' denotes the convolution operation (in the time domain).
\end{theorem} 

See Appendix \ref{pf:residuethm} for the proof. The proof uses the Residue Theorem applied to the equation in Lemma \ref{lem:qspolynomial}. The variables $\tilde{\mathbf{x}}^{\mathrm{natural}}(s)$ and $\mathbf{x}^{\mathrm{natural}}_t$ are referred to as ``natural'' parameters because they both come from nonzero $\mathbf{u}_0$, as opposed to the ``forcing'' term, $-\frac{1}{2}\mathcal{G}\mathcal{G}^T \mathbf{S}_\theta(\mathbf{u}_t, t)$ in Equation \eqref{eq:ode}. It is also noteworthy that the score, $\mathbf{s}_\theta(\mathbf{u}_t, t)$ affects $\mathbf{x}_t$ only through convolution with $h_t^{(n)}$. 

Theorem \ref{thm:residuethm} is a novel result that is most useful for theoretical analysis and may also benefit the design of the forward process itself. It does not directly lead to an algorithm to solve for $\mathbf{x}_t$ because $\mathbf{s}_\theta$ is a function of $\mathbf{u}_t$, which is itself a function of $\mathbf{x}_t$, as well as the fact that $\mathbf{x}^{\mathrm{natural}}_t$ depends on $\mathbf{x}_0$ (which is unknown during the backward process).

Both results speak to the optimality of choosing critically damped parameters. It was proven in \citep{nold} that the critically damped parameters are optimal for the following design objective involving the forward matrix $\mathbf{F}$:
\begin{align}
    \min_{\gamma_2, \gamma_3 \ldots \gamma_{n-1}, \xi} \max \left( \operatorname{Re} \left(\operatorname{eig}(\mathbf{F})\right)\right).
    \label{eq:critdampedobjective}
\end{align}
With Lemma \ref{lem:qspolynomial} and Theorem \ref{thm:residuethm}, one may understand that not choosing critically damped parameters would result in the characteristic polynomial $P(s) = \prod_{i=1}^n (s - s_i)$, which leads to $h_t^{(n)} = \sum_{i=1}^n c_i \exp(s_i t)$ (the $c_i$ are the coefficients obtained from the residue theorem). Because the critically damped parameters minimize \eqref{eq:critdampedobjective}, not using the critically damped parameter choices necessitates the existence of at least one eigenvalue of $\mathbf{F}$, $s_k$ such that $s_k > s_*$. The consequence of such a mode $c_k \exp(s_k t)$ is slower convergence along diffusion time, or from a frequency perspective: a mode that allows more high frequencies from $\mathbf{s}_\theta$ pass into $\mathbf{x}_t$. The critically damped parameters therefore represent the filter with the sharpest frequency cutoff.

A simple derivation in Appendix \ref{pf:OUconv} yields that the solution of the backward ODE associated with the (first-order) Ornstein–Uhlenbeck process: $d\mathbf{x}_t = -\xi \mathbf{x}_t dt + \sqrt{2\xi L^{-1}}d\mathbf{w}_t$ 
may be expressed with $h_t^{\text{OU}} = -\xi L^{-1} \exp(-\xi t)$:
\[\mathbf{x}_t = \exp(-\xi t) \mathbf{x}_0 -\xi L^{-1}\int_0^t \exp(-\xi(t - \tau))\mathbf{s}_\theta(\mathbf{x}_\tau, \tau) d\tau = \exp(-\xi t) \mathbf{x}_0 + h_t^{\text{OU}} * \mathbf{s}_\theta(\mathbf{x}_t, t).\]
Contrasting this result with Theorem~\ref{thm:residuethm} sheds a new light on the inherent regularization of HOLD. It formally shows the smoothing effect of HOLD on the generated sample trajectories, where the
``smoothing'' 
is manifested by
$h_t^{(n)}$ acting on the score function as a low-pass filter with stronger attenuation of high frequencies than $h_t^{\text{OU}}$. 
This function, which constraints the generation trajectories, along with the theory developed in the following section, explains why HOLD reduces memorization.

We turn to visually demonstrate our findings on the filters that convolve with the score function.
Figure \ref{fig:lowpassfilters} 
plots the Fourier transforms of $h_t^{\text{OU}}$ and $h_t^{(n)}$ for $n=2,3,4$, which may be obtained from each Laplace transform by plugging in $s = i\omega$. An ideal low-pass filter takes the form of an indicator function: $H_{\text{ideal}}(\omega) = \begin{cases}
    1, \quad |\omega| \leq \omega_{\text{cutoff}} \\
    0, \quad \text{else}
\end{cases}$; this is achieved by selecting $h_t^{(n)}$ as a scaled $\operatorname{sinc}(t) = \frac{\sin(2 \pi t)}{2\pi t}$. However this cannot be perfectly implemented as a causal $h_t^{(n)}$, resulting in a complex valued Fourier transform, and it is hard to model such a function with a set of differential equations as is being done here. Therefore, the better low-pass filters in Figure \ref{fig:lowpassfilters} allow for a wider range of low frequencies to pass while more sharply attenuating higher frequencies. This sharper attenuation is achieved by the term $t^{n-1}$ in $h^{(n)}_t$, that is associated with a pole in Laplace space of multiplicity $n$. Clearly, the HOLD filters are stronger low-pass filters than the Ornstein–Uhlenbeck filter.

\begin{figure}
    \centering
    \includegraphics[width=0.6\linewidth]{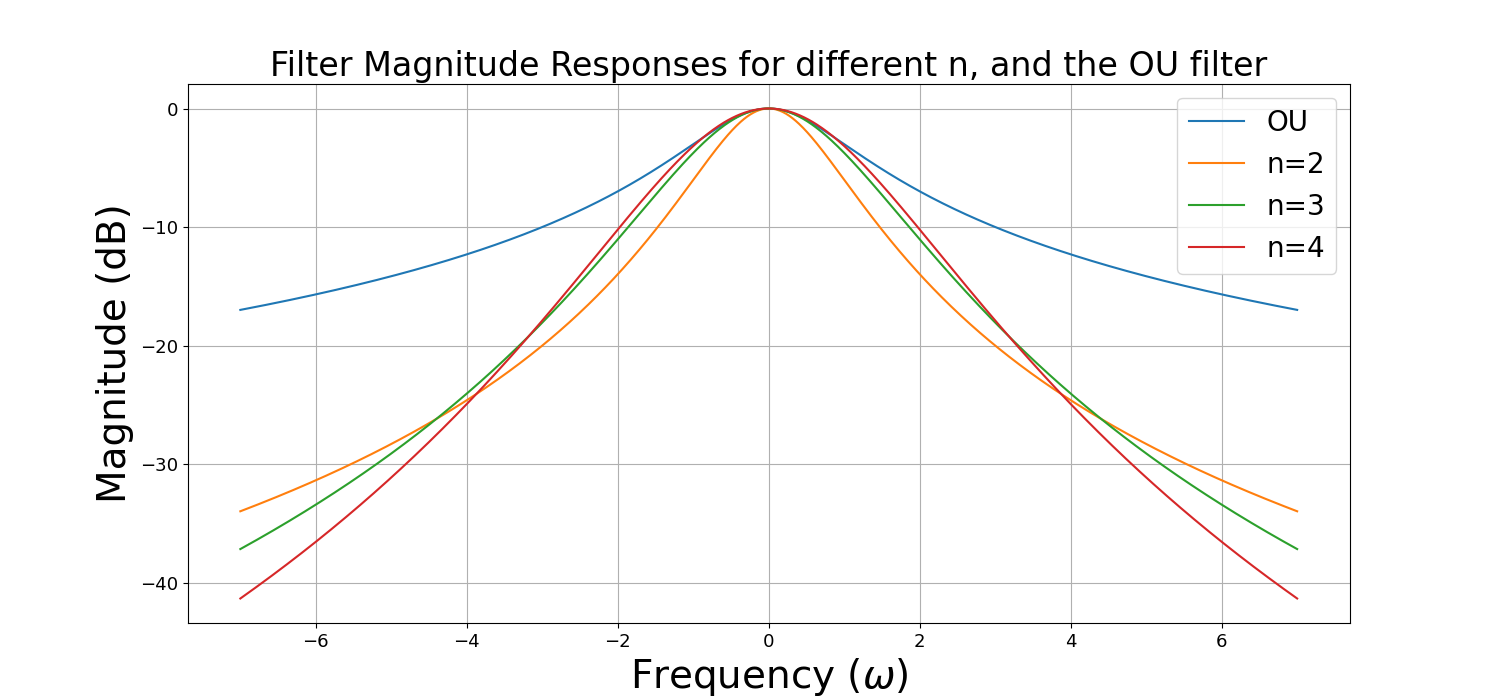}
    \caption{Magnitudes of the Fourier Transforms $|H(i\omega)|$ for different HOLD diffusion model orders $n$, and the Ornstein–Uhlenbeck filter with $\xi = 1$. One may observe that the HOLD filters are better at attenuating higher frequencies than OU, while still allowing a wide band of lower frequencies. 
    }
    \label{fig:lowpassfilters}
\end{figure}

\subsection{HOLD mitigates memorization}

We have shown above that HOLD regularizes the generation process.
In this subsection, we provide reasoning for its reduced memorization (shown empirically to be attenuated and delayed) by adding to our regularized dynamics result the fact that the HOLD optimal empirical score function is more complicated to learn than the optimal empirical score of a standard (first-order) DDM.
Then, under certain approximations, we formally show that HOLD prevents the learned distribution from collapsing to training samples, contrary to the standard Ornstein–Uhlenbeck diffusion process.

\subsubsection{The HOLD optimal empirical score function is hard to learn}

Here, we present the HOLD optimal empirical score function. 
That is, we derive the optimal unconstrained score function that minimizes the HOLD loss objective in Equation \eqref{eq:lossfcn} under empirical data distribution. The derivation makes use of a technical approximation that each training data point gets assigned a set of auxiliary variables once at the algorithm's initialization. A justification for this approximation is provided in subsection \ref{sec:roleinitaux}.

\begin{proposition} \label{prop:empdist}
For the empirical distribution $p(\mathbf{u}_0)=\frac{1}{n_{\text{train}}}\sum\nolimits_{k=1}^{n_{\text{train}}} \delta(\mathbf{u}_0-\mathbf{u}_0^{(k)})$, the distribution of $\mathbf{u}_t$, denoted by $p^{\mathrm{emp}}_t$, obeys
\begin{align*}
    p^{\mathrm{emp}}_t(\mathbf{u}) &= \frac{1}{n_{\mathrm{train}}} \sum_{k=1}^{n_{\mathrm{train}}} (2\pi  \det \boldsymbol{\Sigma}_t)^{-h/2} \exp \left(-\frac{1}{2}\bigg|\bigg|\boldsymbol{\Sigma}_t^{-1/2}\left( \mathbf{u} - \exp(\mathcal{F}t)\mathbf{u}^{(k)}_0\right)\bigg|\bigg|^2 \right) \\
    \nabla_{\mathbf{u}}\log p^{\mathrm{emp}}_t(\mathbf{u}) &= - \boldsymbol{\Sigma}^{-1}_t \mathbf{u} +  \boldsymbol{\Sigma}^{-1}_t \exp(\mathcal{F}t) \frac{\sum_{k=1}^{n_{\mathrm{train}}}\mathcal{N}(\mathbf{u} \mid \exp(\mathcal{F}t)\mathbf{u}^{(k)}_0,\boldsymbol{\Sigma}_t)
    \,\, \mathbf{u}^{(k)}_0}
    {\sum_{k=1}^{n_{\mathrm{train}}}\mathcal{N}(\mathbf{u} \mid \exp(\mathcal{F}t)\mathbf{u}^{(k)}_0,\boldsymbol{\Sigma}_t)}.
\end{align*}
The unconstrained score that minimizes \eqref{eq:lossfcn} is given by 
$\mathbf{s}_{\mathrm{emp}}(\mathbf{u}_t, t) = \left[\nabla_{\mathbf{u}_t} \log p^{\mathrm{emp}}_t(\mathbf{u}_t) \right]_{n(h-1):nh}$.
\end{proposition}

See Appendix \ref{pf:empdist} for the proof. It proceeds in a very similar manner to the proof for the standard diffusion model, but in a multivariate non-isotropic setting. 

This proposition sheds some light on the hypothesized regularization of the training procedure. 
Specifically, note that the empirical distribution's score is a function of each $\mathbf{u}^{(k)}_0$, which contains initial auxiliary variables ($\mathbf{v}_0^{(1),(k)}, \ldots , \mathbf{v}_0^{(n-1),(k)}$) that are completely independent of the training data $\mathbf{x}_0^{(k)}$. HOLD thereby reduces the ability of the score network to fully memorize the true empirical distribution of the training data itself. It does so by forcing the score function to memorize the joint distribution containing the auxiliary variables instead.

\subsubsection{HOLD prevents distribution collapse}

Here, under certain simplifying assumptions, we formally show another distinction between HOLD and the standard first-order diffusion process.
Specifically, we approximate the Mahalanobis distance between the $k$th training sample and the empirical diffusion distribution as time approaches zero. It is shown below that this distance approaches zero for the Ornstein–Uhlenbeck diffusion process as $t \to 0$, but approaches a finite limit for HOLD $n=2$, and goes to infinity for HOLD $n=3$.

\begin{proposition} \label{prop:mahalanobis}
    Consider the $k$th training data sample $\mathbf{x}^{(k)}_t$, that including auxiliary variables becomes $\mathbf{u}^{(k)}_0$. Supposing the training data samples are far enough apart, one may approximate the empirical distributions of the HOLD process for $n=\{2,3\}$ and the Ornstein–Uhlenbeck process as follows:
    \begin{align*}
    p^{\mathrm{emp, HOLD}}_t &\approx \mathcal{N}(\boldsymbol{\mu}^{\mathrm{HOLD}}_t, \boldsymbol{\Sigma}^{\mathrm{HOLD}}_t) := \mathcal{N}(\exp(\mathcal{F} t)\mathbf{u}^{(k)}_0, L^{-1}(\mathbf{I} - \exp(\mathcal{F} t) \exp(\mathcal{F}t)^T )), \\
    p^{\mathrm{emp, OU}}_t &\approx \mathcal{N}(\boldsymbol{\mu}^{\mathrm{OU}}_t, \boldsymbol{\Sigma}^{\mathrm{OU}}_t) := \mathcal{N}(\exp(-\xi t)\mathbf{x}^{(k)}_0, L^{-1}(1 - \exp(-2\xi t))\mathbf{I}).
\end{align*}
Then, the following Mahalanobis distance limits apply:
\begin{align*}
    \lim_{t \to 0^+} D_M(\mathbf{x}^{(k)}_0, p^{\mathrm{emp, OU}}_t) = 0, \quad \lim_{t \to 0^+} D_M(\mathbf{u}^{(k)}_0, p^{\mathrm{emp, HOLD}}_t) \gg 0.
\end{align*}
\end{proposition}
See Appendix \ref{pf:mahalanobis} for the proof. This proposition is proven for the Ornstein–Uhlenbeck process and HOLD model orders $n=2,3$ through direct computation of limits. It suggests that for time $t \to 0$ the Ornstein–Uhlenbeck distribution collapses onto the data, but the HOLD distribution does not. Furthermore, Figure \ref{fig:mahalanobisdists}, numerically computes the determinants of the inverse covariance matrices used in Appendix \ref{pf:mahalanobis}. This figure suggests that the Mahalanobis distances diverge for $n$ larger than $2$. Therefore, HOLD is able to avoid distribution collapse, unlike the Ornstein–Uhlenbeck process.

\begin{figure}
    \centering
    \includegraphics[width=0.6\linewidth]{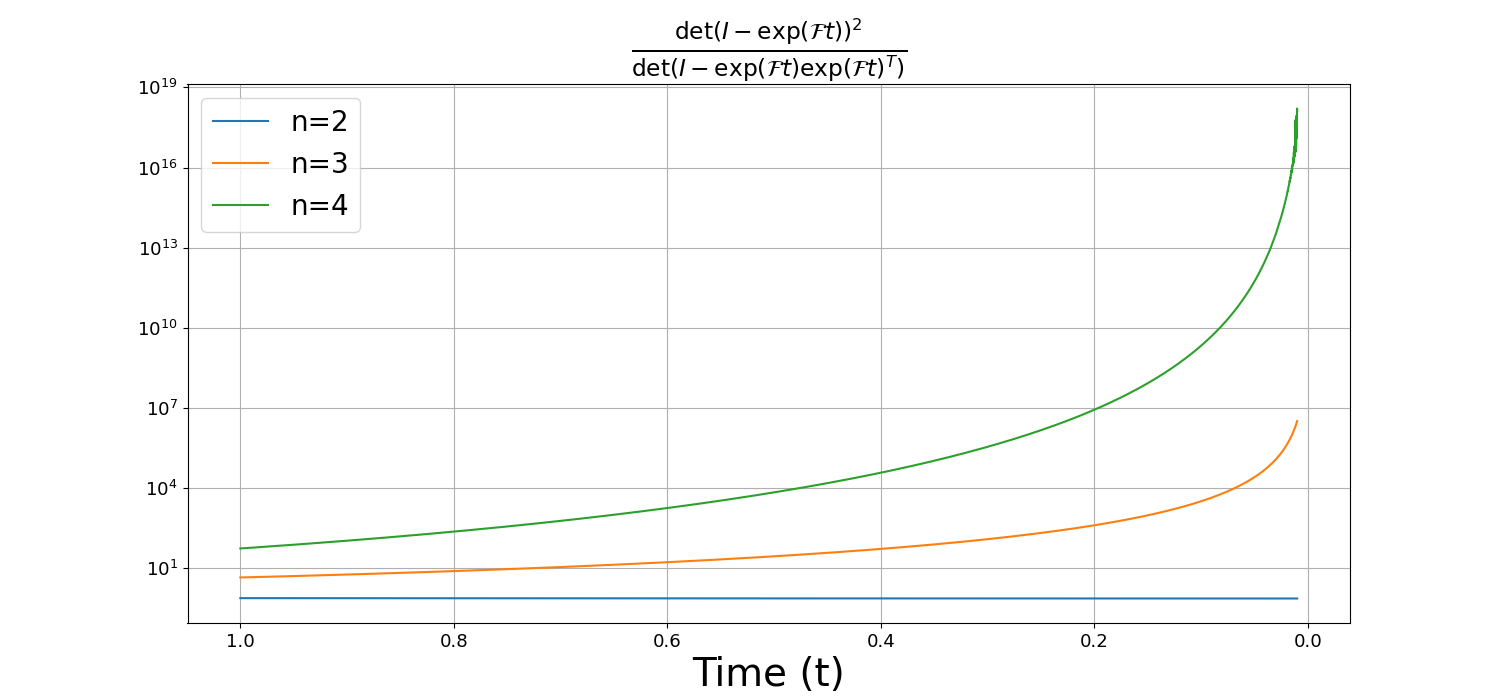}
    \caption{Determinant of the inverse covariance matrix, assuming $L^{-1}=1$, used to calculate the Mahalanobis distance in Proposition \ref{prop:mahalanobis}. As one raises model order, the determinant only increases, thereby improving resistance to distribution collapse.}
    \label{fig:mahalanobisdists}
\end{figure}

\section{Experiments}
\label{sec:experiments}

In order to validate the theory that HOLD regularizes the training and sample generation processes, we perform experiments in a setup similar to \citep{bonnaire2025why}. We compare memorization rate and generation quality along the training of HOLD diffusion models and first-order diffusion models based on the widely used Variance Preserving Stochastic Differential Equation (VPSDE) framework. Successful regularization should force memorization lower while preserving similar image qualities, measured by Fr\'{e}chet Inception Distance (FID) \citep{heusel2017gans}. Memorization measurement is performed on the sample level, as in \citep{bonnaire2025why}. The distances between every generated sample and every training sample are computed. For every generated sample, the gap ratio, the ratio of distances between the first closest and second closest training samples is computed. If it falls below a certain threshold, then that generated sample is declared to be memorized. Each reported memorization percentage is the ratio of memorized generated samples over the total number of generated samples in that batch.

Formally, we use a batch size $B=1024$ and a gap ratio threshold $\tau=0.333$ (as in \citep{bonnaire2025why,yoon2023diffusion,gu2025on}), and define $d^{(j)}_k$ as the $\ell_2$ distance between the $k$th generated sample and the $j$th closest image in the training set. The memorization metric, termed Fmem, is calculated as follows:
\begin{align*}
    M_k &= \begin{cases}
        1, \quad d^{(1)}_k/d^{(2)}_k < \tau \\
        0, \quad \text{else}
    \end{cases}, \quad 
    \text{Fmem} = \frac{1}{B}\sum_{k=1}^B M_k.
\end{align*}
Confidence intervals are obtained recognizing that Fmem is a sample proportion. More experimental details appear in Appendix~\ref{app:miscExpDetails}.

\subsection{CelebA Dataset}

The main experimental setup in our paper examines the memorization behavior of HOLD throughout training on the CelebA dataset \citep{liu2015deep} for different dataset sizes. To facilitate extensive examination, before training, the images are shrunk from the center to size $32 \times 32$, and are converted from RBG to grayscale. These are the exact same preprocessing steps that \citep{bonnaire2025why} uses; we also use the same UNet architecture and data splits. The only architectural deviation was the additional inputs required for the auxiliary variables.

Figure \ref{fig:celebafidfmems} presents the FID and Fmem results during training for four different dataset sizes $n_{\text{train}}$.
From inspection of this figure, HOLD orders $2$ and $3$ respectively improve memorization resistance for comparable FID levels at each selected $n_{\text{train}}$. The generated data samples in Figure \ref{fig:celebamemorization-three-models} confirm these findings (See Appendix~\ref{app:visual_results} for more visual results). Eight generated samples are drawn from each model and compared against their nearest neighbors in the training set. A majority the images generated by the VPSDE diffusion model either closely or identically resemble certain training data images. The HOLD $n=2$ case significantly reduces the memorization percentage and for the most part produces images that look quite different from their nearest training neighbors. Finally HOLD $n=3$ further improves HOLD $n=2$ in the same fashion. Each diffusion model produces nearly equal FIDs at this number of training iterations, and visual image quality is similar overall.

\begin{figure}
    \centering
    \includegraphics[width=1.0\linewidth]{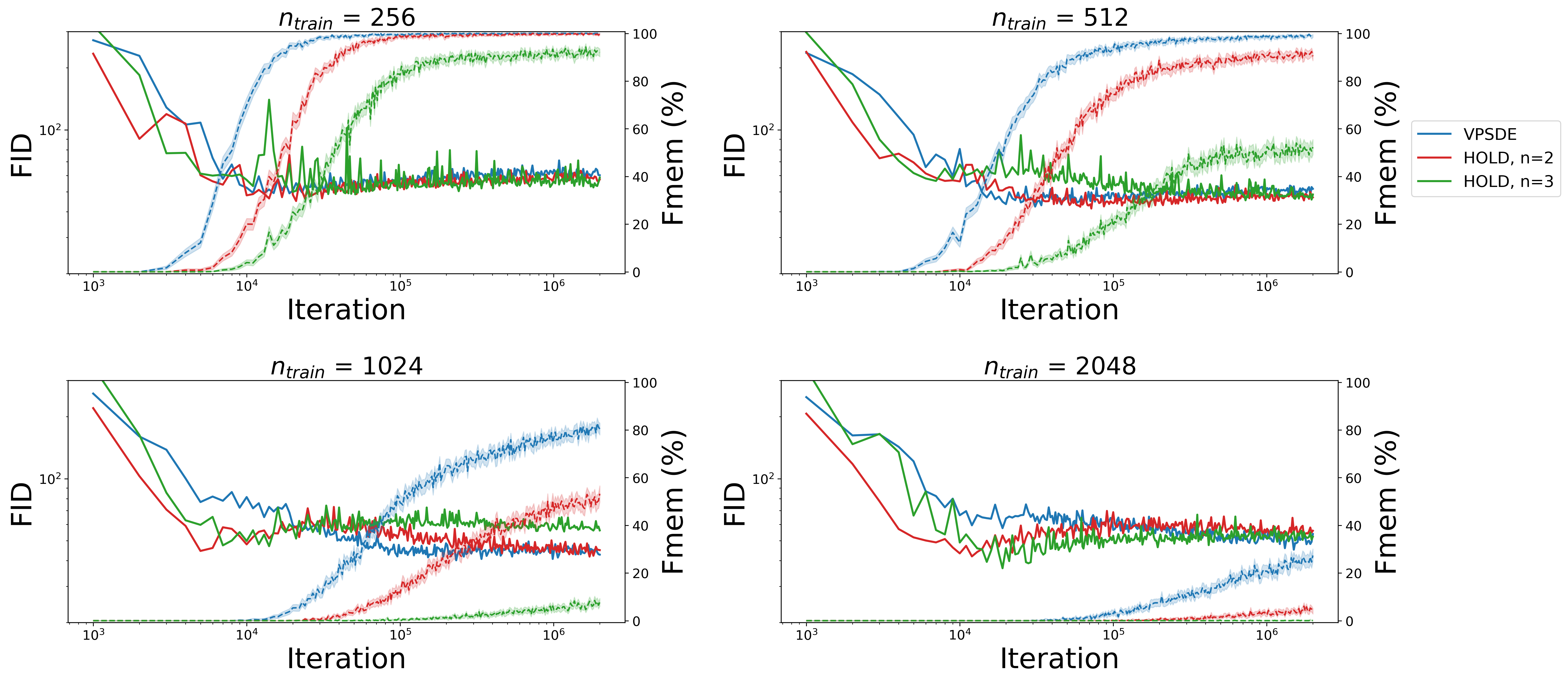}
    \caption{CelebA FIDs and fraction memorized (Fmem) percentages by the number of training samples. Memorization is presented with 95\% confidence intervals. Using HOLD and increasing the model order helps to mitigate memorization for roughly the same FID levels.}
    \label{fig:celebafidfmems}
\end{figure}

\begin{figure}[t]
    \centering

    \begin{subfigure}[b]{\textwidth}
        \centering
        \includegraphics[width=0.6\textwidth]{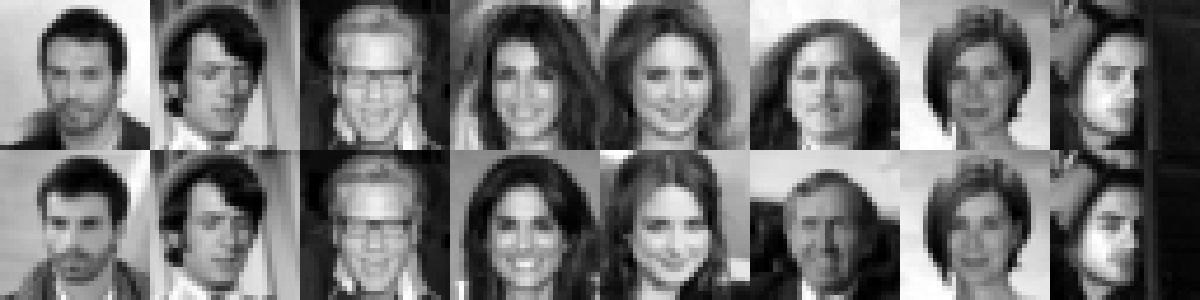}
        \caption{VPSDE. $\text{Fmem: }27.339\%, \text{FID: }53.823$}
        \label{fig:memorization-hold}
    \end{subfigure}

    \begin{subfigure}[b]{\textwidth}
        \centering
        \includegraphics[width=0.6\textwidth]{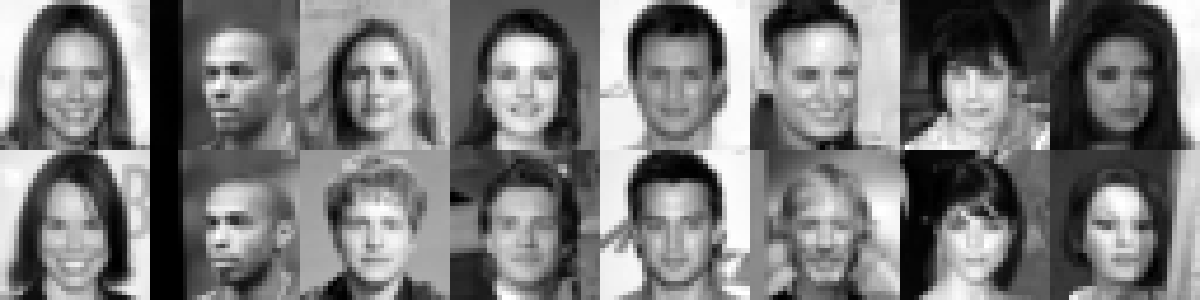}
        \caption{HOLD $n=2$. $\text{Fmem: } 4.029\%,\text{FID: }55.819$}
        \label{fig:memorization-vpsde}
    \end{subfigure}

    \begin{subfigure}[b]{\textwidth}
        \centering
        \includegraphics[width=0.6\textwidth]{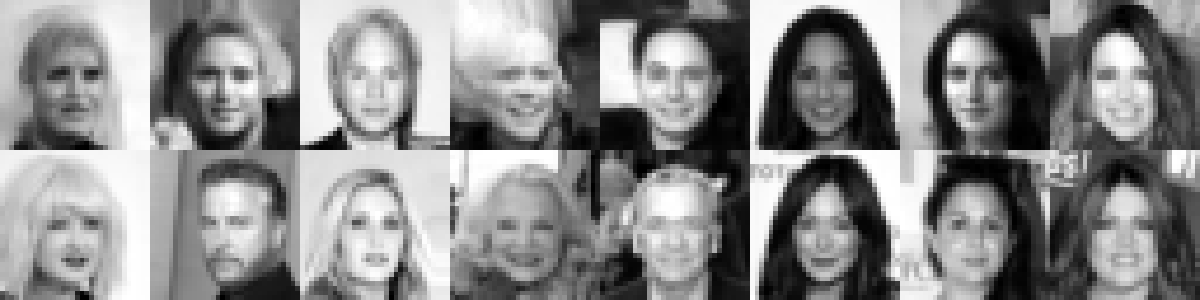}
        \caption{HOLD $n=3$. $\text{Fmem: } 0.101\%,\text{FID: }52.223$}
        \label{fig:memorization-third}
    \end{subfigure}

    \caption{Nearest training neighbors for different models at $2\mathrm{e}6$ training iterations with $2048$ training images on the CelebA dataset. Each first row contains the generated images, and each second row contains the corresponding nearest neighbors. The VPSDE generated samples are heavily memorized, while the HOLD generated images, for the most part, are not nearly as memorized.}
    \label{fig:celebamemorization-three-models}
\end{figure}

\subsection{CIFAR-10 Dataset}

We perform similar experiments with the same UNet architecture on the CIFAR-10 dataset \citep{krizhevsky2009learning} (converted to grayscale), and similar conclusions hold here. 
The key difference between CIFAR-10 and CelebA is that this dataset features ten different image categories, resulting in more dataset diversity and ultimately fewer images to learn from per category. However, this difference does not change the capabilities of HOLD that delay and attenuate memorization. Figure \ref{fig:cifar10fidfmems} demonstrates this, with memorization decreasing for higher orders, and plateauing to roughly the same FID. These curves do exhibit higher memorization variances and higher FIDs, but this is due to the fact that there are less images per category. 

\begin{figure}
    \centering
    \includegraphics[width=0.8\linewidth]{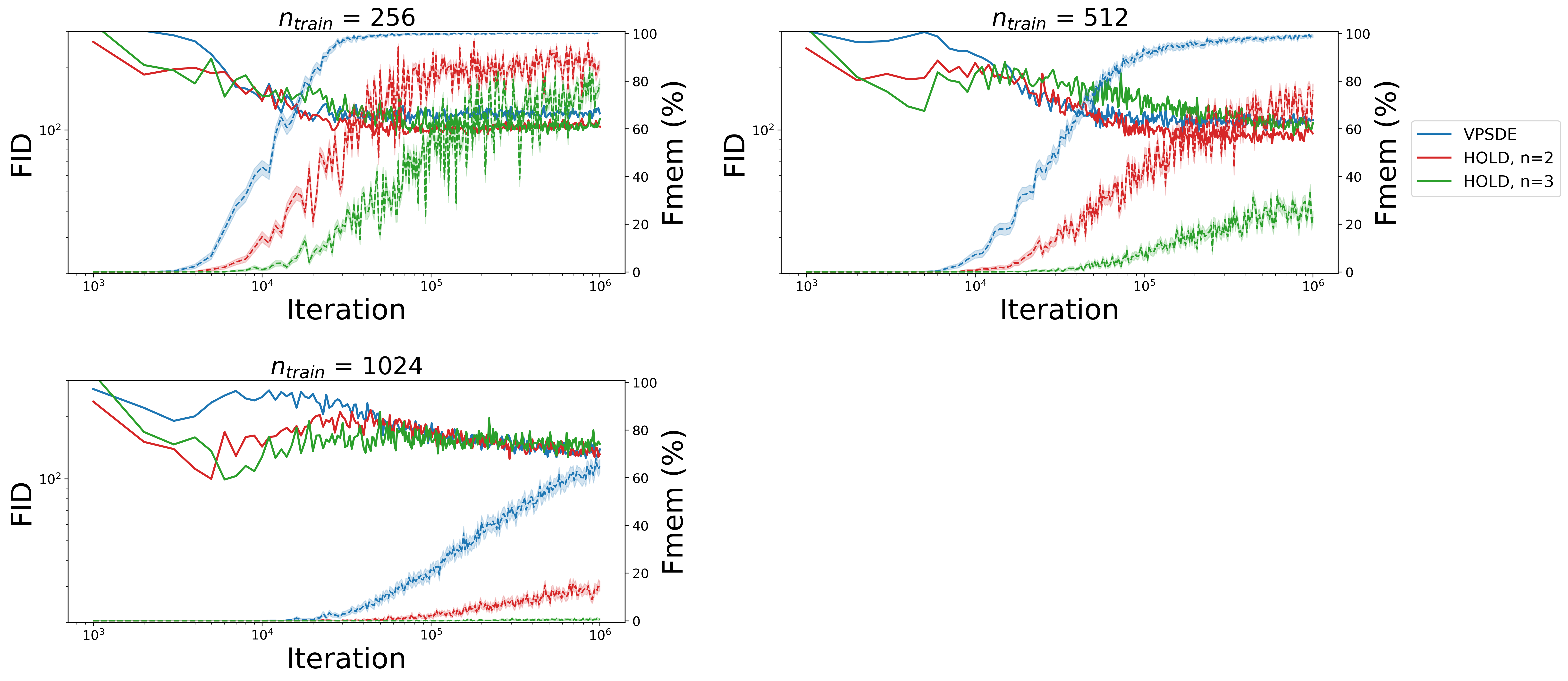}
    \caption{CIFAR-10 FIDs and fraction memorized (Fmem) percentages by the number of training samples. Memorization presented with 95\% confidence intervals. Using HOLD and increasing the model order helps to mitigate memorization for similar FID levels.}
    \label{fig:cifar10fidfmems}
    \vspace{-5mm}
\end{figure}

\section{Conclusion}
\label{sec:conclusion}
While diffusion models offer unparalleled generated image quality compared to earlier methods, they are also susceptible to memorization. Previous works have suggested early stopping as a solution, but when this is inconvenient to do, there are not many other techniques in the literature that are designed to prevent model memorization. The theoretical and empirical results presented in this work suggest that HOLD models may be used to perform statistical shrinkage on both the training and sampling procedures, alleviating this problem that is otherwise difficult to measure. One future point of work would be further exploration of diffusion models that implicitly act as low-pass filters on their score function. This work did not explore optimizing the forward SDE to obtain optimal filtering properties. It would be worthwhile to explore alternate models inspired by classical signal processing, and whether they could potentially further reduce memorization without compromising the generation quality. 

Some limitations of HOLD are that it requires a slightly larger memory footprint. However, the number of necessary parameters do not increase by a factor of $n$ (the model order); in our experiments, only extra inputs were added to the network to include the auxiliary variables. It is also worth noting that $n=4$ still achieves reasonable image qualities, but the qualities start to taper off for model orders higher than this \citep{nold}.
Regarding the broader impact of our work, reducing memorization in diffusion models is largely beneficial: when training data is used with consent, it helps protect copyright and user privacy by making such data harder to extract from the trained model. On the other hand, making diffusion models less prone to memorization may also make it harder to detect when users’ data has been used for training without their consent.

\section*{Acknowledgment}
TT was supported by the Israel Science Foundation grant No.~1940/23 and MOST grant No.~0007091.

\bibliographystyle{plainnat}
\bibliography{refs}

\clearpage

\appendix

\section{Technical Proofs}

This section of the appendix will be used to rigorously prove the theoretical claims of paper.

\subsection{Proof of Proposition \ref{prop:empdist}}
\label{pf:empdist}
\begin{proof}
    Start with training samples $\{ \mathbf{u}_0^{(k)}\}, 1 \leq k \leq n_{\text{train}}$. The data distribution, including initial auxiliary variables, may be expressed as the following linear combinations of Dirac delta distributions:
\begin{align*}
    p_{\text{data}}(\mathbf{u}) &= \frac{1}{n_{\text{train}}}\sum_{k=1}^{n_{\text{train}}} \delta(\mathbf{u} - \mathbf{u}_0^{(k)}) \\
    \mathcal{L}(s_\theta) &= \mathbb{E}_{t\sim \mathcal{U}(0, 1), \mathbf{u}_0, \boldsymbol{\epsilon}_{\text{full}}}|| \boldsymbol{\epsilon}_n + \mathbf{s}_\theta(\mathbf{u}_t, t)\left(\mathbf{L}_t[nh, nh]\right) ||^2 \\
    &= \int\mathbb{E}_{t\sim \mathcal{U}(0, 1), \mathbf{u}_0}|| \boldsymbol{\epsilon}_n + \mathbf{s}_\theta(\mathbf{u}_t, t)\left(\mathbf{L}_t[nh, nh]\right) ||^2 p(\boldsymbol{\epsilon}_{\text{full}}) d\boldsymbol{\epsilon}_{\text{full}} = \int \mathcal{L}(s_\theta, \mathbf{u}_t) d\boldsymbol{\epsilon}_{\text{full}}.
\end{align*}
Minimizing $\mathcal{L}(s_\theta)$ may be achieved by minimizing $\mathcal{L}(s_\theta, \mathbf{u}_t)$ instead. We do this by minimizing this loss over all $\mathbf{s}_\theta$ as follows. We define $\boldsymbol{\epsilon}_{\text{full}}^{(k)} = \mathbf{L}_t^{-1}\left(\mathbf{u}_t - \exp(\mathcal{F}t)\mathbf{u}_0^{(k)}\right)$ and $\boldsymbol{\epsilon}_n^{(k)} = \left[\boldsymbol{\epsilon}^{(k)}_{\text{full}} \right]_{n(h-1):nh}$; that is the final n entries of the vector $\boldsymbol{\epsilon}^{(k}_{\text{full}}$. Both come from the $k$th sample.
\begin{align*}
    \mathcal{L}(\mathbf{s}_\theta, \mathbf{u}_t) &= \frac{1}{n_{\text{train}}} \sum_{k=1}^{n_{\text{train}}} || \boldsymbol{\epsilon}^{(k)}_n + \mathbf{s}_\theta(\mathbf{u}_t, t) \left( \mathbf{L}_t[nh, nh] \right)||^2 p(\boldsymbol{\epsilon}_{\text{full}}^{(k)}).
\end{align*}
Take the gradient with respect to $\mathbf{s}_\theta$ and set it to zero:
\begin{align*}
    &\sum_{k=1}^{n_{\text{train}}} \left(\boldsymbol{\epsilon}^{(k)}_n + \mathbf{s}_\theta(\mathbf{u}_t, t) \left( \mathbf{L}_t[nh, nh] \right)\right) p(\boldsymbol{\epsilon}_{\text{full}}^{(k)}) = 0 \\
    &\mathbf{s}_\theta(\mathbf{u}_t, t) = -\frac{\sum_{k=1}^{n_{\text{train}}} p(\boldsymbol{\epsilon}^{(k)}_{\text{full}}) \boldsymbol{\epsilon}_n^{(k)} / \mathbf{L}_t[nh, nh]}{\sum_{k=1}^{n_{\text{train}}} p(\boldsymbol{\epsilon}^{(k)}_{\text{full}})} = \frac{\sum_{k=1}^{n_{\text{train}}} p(\boldsymbol{\epsilon}^{(k)}_{\text{full}}) \nabla_{\mathbf{v}^{(n-1)}_t}\log p(\mathbf{u}_t \mid \mathbf{u}^{(k)}_0)}{\sum_{k=1}^{n_{\text{train}}} p(\boldsymbol{\epsilon}^{(k)}_{\text{full}})}\\
    &= \left[ \frac{\sum_{k=1}^{n_{\text{train}}} p(\boldsymbol{\epsilon}^{(k)}_{\text{full}}) \nabla_{\mathbf{u}_t}\log p(\mathbf{u}_t \mid \mathbf{u}^{(k)}_0) }{\sum_{k=1}^{n_{\text{train}}} p(\boldsymbol{\epsilon}^{(k)}_{\text{full}})} \right]_{n(h-1):nh}.
\end{align*}
If one calculates the score function corresponding to $p_t^{\text{emp}}$, then it follows that $\mathbf{s}_\theta(\mathbf{u}_t, t) = \left[\nabla_{\mathbf{u}_t} \log p^{\text{emp}}_t(\mathbf{u}_t) \right]_{n(h-1):nh}$.
\end{proof}

\subsection{Proof of Lemma \ref{lem:qspolynomial}}
\label{pf:qspolynomial}
\begin{proof}
Note that $\tilde{\mathbf{u}}(s) = \operatorname{vec}(\tilde{\mathbf{x}}(s), \tilde{\mathbf{v}}^{(1)}(s), \ldots \tilde{\mathbf{v}}^{(n-1)}(s)), \Tilde{\mathbf{S}}(s) = \mathbf{S}_\theta(\mathbf{u}_t, t)$.
\begin{align*}
    \frac{d\mathbf{u}_t}{dt} &= \mathcal{F}\mathbf{u}_t - \xi L^{-1} \mathbf{E}_{n,n} \mathbf{S}_\theta(\mathbf{u}_t, t) \\
    s\tilde{\mathbf{u}}(s) - \mathbf{u}_0 &= \mathcal{F}\tilde{\mathbf{u}}(s) - \xi L^{-1} \mathbf{E}_{n,n} \tilde{\mathbf{S}}(s) \\
    ((\mathbf{F} - s\mathbf{I}) \otimes \mathbf{I}_h )\tilde{\mathbf{u}}(s) &= \xi L^{-1} \mathbf{E}_{n,n} \tilde{\mathbf{S}}(s) - \mathbf{u}_0.
\end{align*}
Now, use Cramer's Rule to solve for $\tilde{\mathbf{x}}(s)$. Take $\mathbf{0}, \mathbf{1} \in \mathbb{R}^h$ as the vector of all zeros and ones respectively. $\mathrm{blockdet}$ takes determinants across each coordinate of the following vectors and returns them in a single vector. $P(s) = \det(\mathbf{F}-s\mathbf{I})$ is the characteristic polynomial of $\mathbf{F}$. The second line follows from the multilinearity of the determinant.
\begin{align*}
    \tilde{\mathbf{x}}(s) &= \frac{1}{P(s)} \mathrm{blockdet} \begin{pmatrix}
        -\mathbf{x}_0 & \gamma_1 \mathbf{1} & \ldots & \mathbf{0}\\
        -\mathbf{v}^{(1)}_0 & -s \mathbf{1} & \ldots & \mathbf{0} \\
        \ldots & \ldots & \ldots & \ldots \\
        -\mathbf{v}^{(n-2)}_0 & \ldots & -s\mathbf{1} & \gamma_{n-1}\mathbf{1} \\
        \xi L^{-1}\tilde{\mathbf{s}}(s) - \mathbf{v}^{(n-1)}_0 & \ldots & -\gamma_{n-1}\mathbf{1} & (-s-\xi)\mathbf{1}
    \end{pmatrix}\\
    &= \frac{\xi L^{-1}}{P(s)} \mathrm{blockdet} \begin{pmatrix}
        \mathbf{0} & \gamma_1 \mathbf{1} & \ldots & \mathbf{0}\\
        \mathbf{0} & -s\mathbf{1} & \ldots & \mathbf{0} \\
        \ldots & \ldots & \ldots & \ldots \\
        \mathbf{0} & \ldots & -s\mathbf{1} & \gamma_{n-1}\mathbf{1} \\
        \tilde{\mathbf{s}}(s) & \ldots & -\gamma_{n-1}\mathbf{1} & (-s-\xi)\mathbf{1}
    \end{pmatrix} \\
    &+ \frac{1}{P(s)} \mathrm{blockdet} \begin{pmatrix}
        -\mathbf{x}_0 & \gamma_1 \mathbf{1} & \ldots & \mathbf{0}\\
        -\mathbf{v}^{(1)}_0 & -s\mathbf{1} & \ldots &\mathbf{0} \\
        \ldots & \ldots & \ldots & \ldots \\
        -\mathbf{v}^{(n-2)}_0 & \ldots & -s\mathbf{1} & \gamma_{n-1}\mathbf{1} \\
        -\mathbf{v}^{(n-1)}_0 & \ldots & -\gamma_{n-1}\mathbf{1} & (-s-\xi)\mathbf{1}
    \end{pmatrix}\\
    &= \frac{(-1)^{n-1}\xi L^{-1} \tilde{\mathbf{s}}(s)}{P(s)} \det \begin{pmatrix}
        \gamma_1 & 0 &\ldots & 0 \\
        -s & \gamma_2 & \ldots & 0 \\
        \ldots & \ldots & \ldots & \ldots \\
        \ldots & 0 & -s & \gamma_{n-1}
    \end{pmatrix} + \tilde{\mathbf{x}}^{\text{natural}}(s)\\
    \tilde{\mathbf{x}}(s) &= - \frac{\Bar{\gamma} \xi L^{-1}\tilde{\mathbf{s}}(s)}{P(s)} + \tilde{\mathbf{x}}^{\text{natural}}(s).
\end{align*}

where $\tilde{\mathbf{x}}^{\mathrm{natural}}(s)$ is given by:
\[\tilde{\mathbf{x}}^{\mathrm{natural}}(s) = \frac{1}{P(s)} \mathrm{blockdet} \begin{pmatrix}
        -\mathbf{x}_0 & \gamma_1 \mathbf{1} & \ldots &  \mathbf{0}\\
        -\mathbf{v}^{(1)}_0 & -s\mathbf{1} & \ldots & \mathbf{0} \\
        \ldots & \ldots & \ldots & \ldots \\
        -\mathbf{v}^{(n-2)}_0 & \ldots & -s\mathbf{1} & \gamma_{n-1}\mathbf{1} \\
        -\mathbf{v}^{(n-1)}_0 & \ldots & -\gamma_{n-1}\mathbf{1} & (-s-\xi)\mathbf{1}
    \end{pmatrix}.\]

The last determinant is simply the product of gammas because the matrice's main diagonal contains the gammas, and one minor row filled with $-s$. Note that the negative characteristic polynomial is also a valid characteristic polynomial, but we take the polynomial with a leading positive coefficient.
\end{proof}

\subsection{Proof of Theorem \ref{thm:residuethm}}
\label{pf:residuethm}
\begin{proof}
According to Lemma \ref{lem:qspolynomial},
\begin{align*}
    \tilde{\mathbf{x}}(s) &= - \frac{\Bar{\gamma} \xi L^{-1}\tilde{\mathbf{s}}(s)}{(s - s_*)^n} + \tilde{\mathbf{x}}^{\text{natural}}(s).
\end{align*}

This equation assumes that we use the critically-damped parameter selection. Further define $H(s) = -\frac{\Bar{\gamma}\xi L^{-1}}{(s-s_*)^n}$. To solve for $h_t$, we must take the inverse Laplace transform:
\begin{align*}
    h_t &= \frac{1}{2 \pi i}\lim_{T \to \infty} \int_{\gamma - iT}^{\gamma + iT} H(s) \exp(st) ds
\end{align*}
for some $\gamma > 0$. Now we utilize the Residue Theorem:
\begin{align*}
    h_t &= \operatorname{Res} (H(s) \exp(st))\mid_{s = s_*}\\
    &= \lim_{s \to s_*} \frac{d^{n-1}}{ds^{n-1}} \left(H(s) \exp(st)(s - s_*)^n\right) \\
    &= -\Bar{\gamma}\xi L^{-1} \lim_{s \to s_*} \frac{d^{n-1}}{ds^{n-1}} \left(\exp(st)\right) \\
    &= -\Bar{\gamma}\xi L^{-1} t^{n-1} \exp(s_*t) = -\Bar{\gamma}n\sqrt{2n-3}L^{-1}t^{n-1}\exp(-t\sqrt{2n-3}).
\end{align*}
Finally, the Convolution Theorem applied to $\tilde{\mathbf{x}}(s) = H(s)\tilde{\mathbf{s}}(s) + \tilde{\mathbf{x}}^{\text{natural}}(s)$ proves the theorem.
\end{proof}

\subsection{Proof of Proposition \ref{prop:mahalanobis}}
\label{pf:mahalanobis}

\begin{proof}
    \begin{align*}
    p^{\text{emp, HOLD}}_t &\approx \mathcal{N}(\boldsymbol{\mu}^{\text{OU}}_t, \boldsymbol{\Sigma}^{\text{HOLD}}_t) = \mathcal{N}(\exp(\mathcal{F} t)\mathbf{u}^{(k)}_0, L^{-1}(\mathbf{I} - \exp(\mathcal{F} t) \exp(\mathcal{F}t)^T )), \\
    p^{\text{emp, OU}}_t &\approx \mathcal{N}(\boldsymbol{\mu}^{\text{OU}}_t, \boldsymbol{\Sigma}^{\text{OU}}_t) = \mathcal{N}(\exp(-\xi t)\mathbf{x}^{(k)}_0, L^{-1}(1 - \exp(-2\xi t))\mathbf{I}).
\end{align*}
One may formally calculate that the Mahalanobis distance for the Ornstein–Uhlenbeck diffusion process is zero. Use L'H\^{o}pital's rule to solve the final limit.
\begin{align*}
    D_{M}(\mathbf{x}^{(k)}_0, p^{\text{emp, OU}}_t)^2 &= (\mathbf{x}^{(k)}_0 - \exp(-\xi t)\mathbf{x}^{(k)}_0)^T \frac{\mathbf{I}}{L^{-1}(1 - \exp(-2\xi t))}(\mathbf{x}^{(k)}_0 - \exp(-\xi t)\mathbf{x}^{(k)}_0) \\
    &= \frac{(1 - \exp(-\xi t))^2}{L^{-1}(1 - \exp(-2\xi t))}||\mathbf{x}^{(k)}_0||^2, \\
    \lim_{t \to 0^+} D_{M}(\mathbf{x}^{(k)}_0, p^{\text{emp, OU}}_t)^2 &= \lim_{t \to 0^+}\frac{(1 - \exp(-\xi t))^2}{L^{-1}(1 - \exp(-2\xi t))}||\mathbf{x}^{(k)}_0||^2 = 0.
\end{align*}

However, this limit is different for HOLD:
\begin{align*}
    D_{M}(\mathbf{u}^{(k)}_0, p^{\text{emp, HOLD}}_t)^2 &= \frac{1}{L^{-1}}(\mathbf{u}^{(k)}_0 - \exp(\mathcal{F}t)\mathbf{u}^{(k)}_0)^T \left(\mathbf{I} - \exp(\mathcal{F}t)\exp(\mathcal{F}t)^T\right)^{-1} (\mathbf{u}^{(k)}_0 - \exp(\mathcal{F} t)\mathbf{u}^{(k)}_0) \\
    &= \frac{1}{L^{-1}}(\mathbf{u}^{(k)}_0)^T(\mathbf{I} - \exp(\mathcal{F}t)^T)\left(\mathbf{I} - \exp(\mathcal{F}t)\exp(\mathcal{F}t)^T\right)^{-1} (\mathbf{I} - \exp(\mathcal{F} t))\mathbf{u}^{(k)}_0.
\end{align*}
We are interested in the limit behavior as $t \to 0^+$. To analyze it, take the limit of the determinant of the inverse covariance matrix in the middle. For the purposes of this determinant, use $\mathbf{F}$ instead of $\mathcal{F} = \mathbf{F} \otimes \mathbf{I}_h$. The proof works out identically since for any matrix $\mathbf{A} \in \mathbb{R}^{n \times n}$, $\exp(\mathbf{A} \otimes \mathbf{I}_h) = \exp(\mathbf{A}) \otimes \mathbf{I}_h$, and $\det(\mathbf{A} \otimes \mathbf{I}_h) = \det(\mathbf{A})^h$; this does not change the determinant's behavior as we mainly care whether it goes to zero, remains finite, or diverges to infinity.
\begin{align*}
    &\lim_{t \to 0^+} \det \left((\mathbf{I} - \exp(\mathbf{F}t)^T)\left(\mathbf{I} - \exp(\mathbf{F}t)\exp(\mathbf{F}t)^T\right)^{-1} (\mathbf{I} - \exp(\mathbf{F} t)) \right) \\
    &= \lim_{t \to 0^+}\frac{\det(\mathbf{I} - \exp(\mathbf{F}t))^2}{\det(\mathbf{I} - \exp(\mathbf{F}t)\exp(\mathbf{F}t)^T)}.
\end{align*}

Firstly, a property of determinants is that it is the product of the eigenvalues of the argument. The eigenvalues of $\mathbf{I} - \exp(\mathbf{F}t)$ are all $1 - \exp(s_*t)$, therefore $\det(\mathbf{I} - \exp(\mathbf{F}t)) = (1-\exp(s_*t))^n$. Furthermore, as $t \to 0^+$:
\begin{align*}
    (1 - \exp(s_* t))^n &= (1 - (1 - s_* t + \mathcal{O}(t^2)))^n = (s_* t + \mathcal{O}(t^2))^n = (s_* t)^n + \mathcal{O}(t^{n+1}) \\
    \det(\mathbf{I} - \exp(\mathbf{F}t))^2 &= (s_* t)^{2n} + \mathcal{O}(t^{2n+1}) = (2n-3)^n t^{2n} + \mathcal{O}(t^{2n+1}).
\end{align*}

\subsubsection{Case $n=2$}
Specifically for $n=2$, by Equation \eqref{eq:expFt}, $\exp(\mathbf{F}t) = \exp(-t)\left(\mathbf{I} + (\mathbf{F} - s_* \mathbf{I})t \right)$, and $s_* = -1, \xi = 2$. Therefore
\begin{align*}
    \det(\mathbf{I} - \exp(\mathbf{F}t)) &= (1 - \exp(-t))^2 = (1 - (1 - t + \mathcal{O}(t^2)))^2 = (t + \mathcal{O}(t^2))^2 = t^2 + \mathcal{O}(t^3)\\
    \det(\mathbf{I} - \exp(\mathbf{F}t))^2 &= t^4 + \mathcal{O}(t^5).
\end{align*}
Now moving onto the denominator.
\begin{align*}
    \exp(\mathbf{F}t) &= \exp(-t)(\mathbf{I}(1 +t) + \mathbf{F}t) \\
    &= \exp(-t)\begin{pmatrix}
        1+t & -t \\
        t & 1 - t
    \end{pmatrix}
\end{align*}
\begin{align*}
    \exp(\mathbf{F}t)\exp(\mathbf{F}t)^T &= \exp(-2t)\begin{pmatrix}
        1+t & -t \\
        t & 1 - t
    \end{pmatrix}\begin{pmatrix}
        1+t & t \\
        -t & 1 - t
    \end{pmatrix} \\
    &= \exp(-2t) \begin{pmatrix}
        2t^2+2t+1 & 2t^2 \\
        2t^2 & 2t^2-2t+1
    \end{pmatrix} \\
    \mathbf{I} - \exp(\mathbf{F}t)\exp(\mathbf{F}t)^T&= \begin{pmatrix}
        1 - \exp(-2t)(2t^2+2t+1) & -2t^2\exp(-2t) \\
        -2t^2\exp(-2t) & 1 - \exp(-2t)(2t^2-2t+1)
    \end{pmatrix} \\
    \det(\mathbf{I} - \exp(\mathbf{F}t)\exp(\mathbf{F}t)^T) &= (1 - \exp(-2t)(2t^2+2t+1))(1 - \exp(-2t)(2t^2-2t+1)) - 4t^4\exp(-4t).
\end{align*}

Take the rest in pieces. $4t^4\exp(-4t) = 4t^4 + \mathcal{O}(t^5)$, and
\begin{align*}
    1 - \exp(-2t)(2t^2+2t+1) &= 1 - (1 - 2t + 2t^2 - \frac{8t^3}{6} +\mathcal{O}(t^4))(2t^2+2t+1) \\
    &= 1 - (2t^2+2t+1) + 2t(2t^2+2t+1) - 2t^2(2t^2+2t+1) + \frac{8t^3}{6} +\mathcal{O}(t^4) \\
    &= \frac{4t^3}{3} +\mathcal{O}(t^4) \\
    \\
    1 - \exp(-2t)(2t^2-2t+1) &= 1 - (1 - 2t + 2t^2 + \mathcal{O}(t^3))(2t^2-2t+1) \\
    &= 1 - \left( 2t^2 - 2t + 1 + 4t^2 -2t + 2t^2 + \mathcal{O}(t^3)\right) \\
    &= -2t^2 + 2t - 4t^2 + 2t - 2t^2 + \mathcal{O}(t^3) \\
    &= 4t + \mathcal{O}(t^2).
\end{align*}

Finally, one may derive:
\begin{align*}
    \lim_{t \to 0^+}\frac{\det(\mathbf{I} - \exp(\mathbf{F}t))^2}{\det(\mathbf{I} - \exp(\mathbf{F}t)\exp(\mathbf{F}t)^T)} &= \lim_{t \to 0^+} \frac{t^4}{\frac{4t^3}{3}4t-4t^4} = \frac{3}{4}.
\end{align*}

The inverse covariance matrix is therefore positive definite as $t \to 0^+$, so $\lim_{t \to 0^+} D_{M}(\mathbf{u}^{(k)}_0, p^{\text{emp, HOLD}}_t) \gg 0$.
\end{proof}

\subsubsection{Case $n=3$}

\begin{align*}
    \exp(\mathbf{F}t) &= \exp(-t\sqrt{3})\begin{pmatrix}
        t^2 + t\sqrt{3} + 1 & t^2\sqrt{3} + t & t^2\sqrt{2}\\
        -t^2\sqrt{3}-t & -3t^2 + t\sqrt{3} + 1 & -t^2\sqrt{6} + 2t\sqrt{2} \\
        t^2\sqrt{2} & t^2\sqrt{6} - 2t\sqrt{2} & 2t^2 - 2t\sqrt{3}+1
    \end{pmatrix} 
\end{align*}
\begin{align*}
    (1 - \exp(-\sqrt{3}t))^3 &= (1 - (1 - \sqrt{3}t + \mathcal{O}(t^2)))^3 = (\sqrt{3}t + \mathcal{O}(t^2))^3 = 3^{3/2}t^3 + \mathcal{O}(t^4) \\
    \det(\mathbf{I} - \exp(\mathbf{F}t))^2 &= (1 - \exp(-\sqrt{3}t))^6 = (3^{3/2}t^3 + \mathcal{O}(t^4))^2 = 27t^6 + \mathcal{O}(t^7).
\end{align*}
The following was computed with python's sympy library:
\begin{align*}
    \det(\mathbf{I} - \exp(\mathbf{F}t)\exp(\mathbf{F}t)^T) &=\bigg((-36t^4 + 24\sqrt{3}t^3 - 36t^2 - 3)\exp(4\sqrt{3}t) \\ 
    &+ (36t^4 + 24\sqrt{3}t^3 + 36t^2 + 3)\exp(2\sqrt{3}t) + \exp(6\sqrt{3}t) - 1 \bigg)\exp(-6\sqrt{3}t) \\
    &= \frac{24\sqrt{3}}{5} t^9 + \mathcal{O}(t^{10})
\end{align*}
\begin{align*}
    \lim_{t \to 0^+}\frac{\det(\mathbf{I} - \exp(\mathbf{F}t))^2}{\det(\mathbf{I} - \exp(\mathbf{F}t)\exp(\mathbf{F}t)^T)} = \lim_{t \to 0^+}\frac{27t^6}{\frac{24\sqrt{3}}{5}t^9} \to \infty.
\end{align*}

The determinant of the inverse covariance matrix grows infinitely large as $t \to 0^+$, so $\lim_{t \to 0^+} D_{M}(\mathbf{u}^{(k)}_0, p^{\text{emp, HOLD}}_t) \to \infty$.

\subsection{Derivation of Ornstein–Uhlenbeck Convolution}
\label{pf:OUconv}
The Ornstein–Uhlenbeck convolution formula is derived from the deterministic ODE as follows
\begin{align*}
    d\mathbf{x}_t &= \left(-\xi \mathbf{x}_t - \frac{2\xi L^{-1}}{2} \mathbf{s}_\theta(\mathbf{x}_t, t)\right) dt \\
    \frac{d\mathbf{x}_t}{dt} + \xi \mathbf{x}_t &= -\xi L^{-1}\mathbf{s}_\theta(\mathbf{x}_t, t) \\
    \frac{d}{dt}\left( \exp(\xi t)\mathbf{x}_t\right) &= -\xi L^{-1}\exp(\xi t)\mathbf{s}_\theta(\mathbf{x}_t, t) \\
    \mathbf{x}_t &= \exp(-\xi t) \mathbf{x}_0 - \xi L^{-1} \int_0^\tau \exp(-\xi (t - \tau))\mathbf{s}_\theta(\mathbf{x}_\tau, \tau) d\tau.
\end{align*}

\section{Mini-Experiments}

\subsection{Role of Auxiliary Variable Initialization on Memorization}
\label{sec:roleinitaux}

Figure \ref{fig:initaux} compares memorization and image quality on the modified CelebA dataset for the case that initial auxiliary variables are drawn once and assigned to each data sample, and for the case that initial auxiliary variables are drawn for each run. The latter method models the joint diffusion process more closely as each initial auxiliary variable is drawn from the true distribution. Assigning each data point an auxiliary variable could also promote further memorization; analyzing this possibility is the purpose of this experiment. However, Figure \ref{fig:initaux} suggests that whether auxiliary variables are initialized once or constantly during training hardly makes a difference. The difference is slightly larger for model order $n=3$, but still not significant.

\begin{figure}
    \centering
    \begin{subfigure}[b]{0.48\linewidth}
        \centering
        \includegraphics[width=\linewidth]{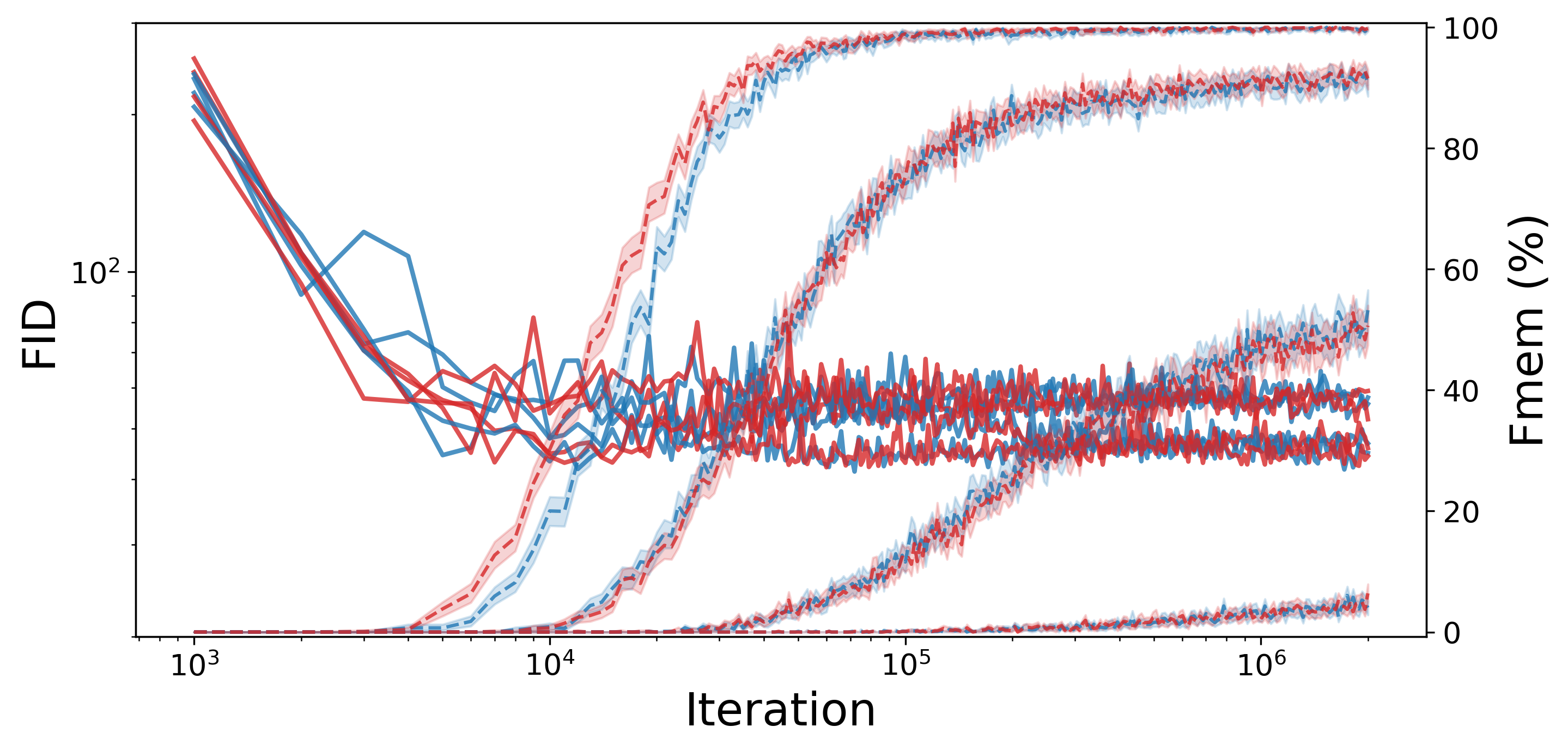}
        \caption{Initial auxiliary variable comparison for $n=2$.}
        \label{fig:initaux2}
    \end{subfigure}
    \hfill
    \begin{subfigure}[b]{0.48\linewidth}
        \centering
        \includegraphics[width=\linewidth]{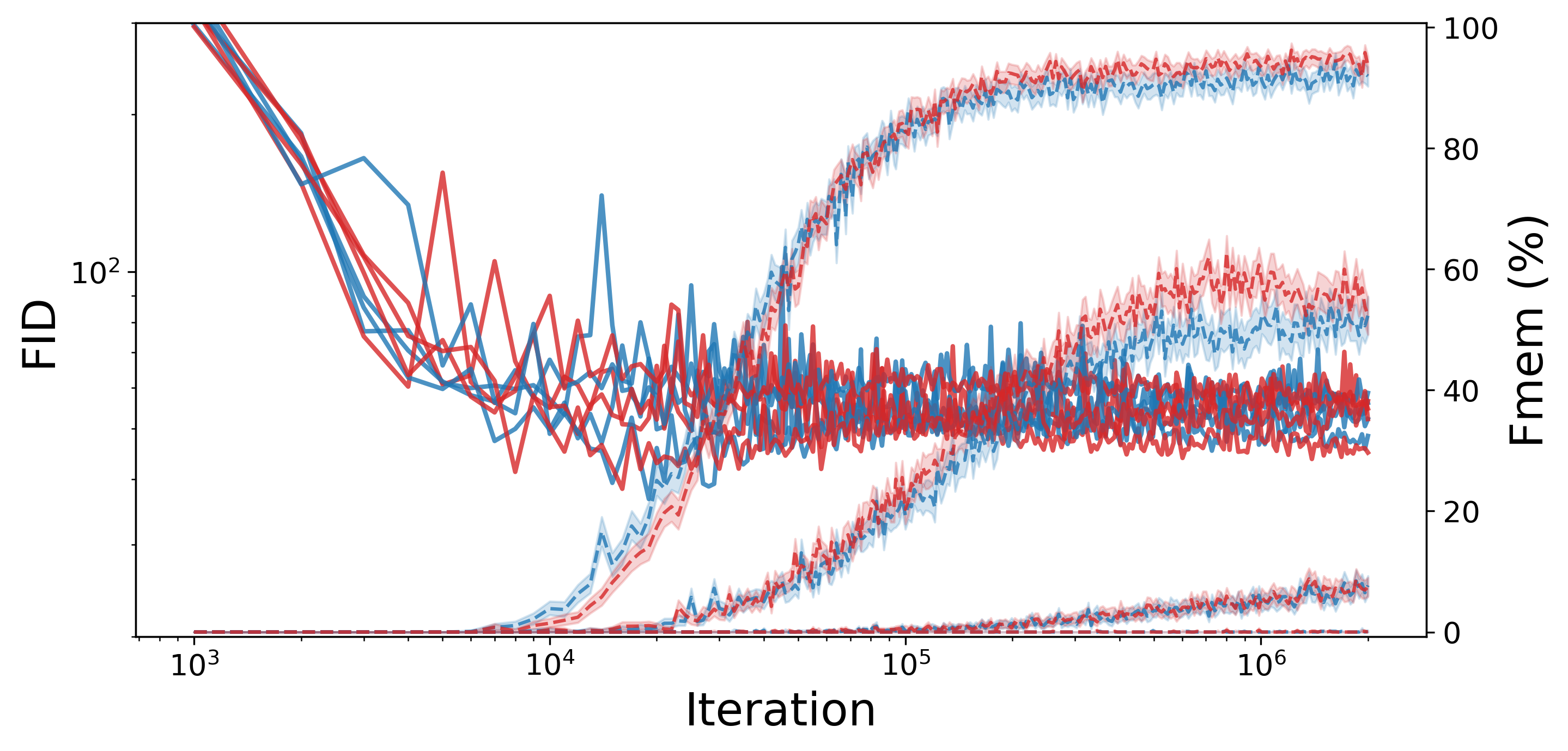}
        \caption{Initial auxiliary variable comparison for $n=3$.}
        \label{fig:initaux3}
    \end{subfigure}
    \caption{Initial auxiliary variable comparisons for $n=2, 3$. Memorization is shown with 95\% confidence intervals. The red curves represent initial auxiliary variables being assigned to each data sample, whereas the blue curves represent initial auxiliary variables being newly drawn for each training iteration. There are four pairs of memorization and FID that correspond to $n_{\text{train}} = 256, 512, 1024, 2048$.}
    \label{fig:initaux}
\end{figure}



\subsection{CIFAR-10 Nearest Neighbors}

Figure \ref{fig:cifar10memorization-three-models}, similarly to Figure \ref{fig:celebamemorization-three-models}, demonstrates that higher order dynamics are less prone to memorization by plotting generated images on the top rows with each corresponding nearest neighbor on the bottom rows. The non-memorized samples are visually lower quality than those of CelebA because of the CIFAR-10 category issue, and the fact these images were taken after only $1{,}000{,}000$ training iterations; these experiments were shortened due to time constraints.

\begin{figure}[t]
    \centering
    \begin{subfigure}[b]{\textwidth}
        \centering
        \includegraphics[width=0.6\textwidth]{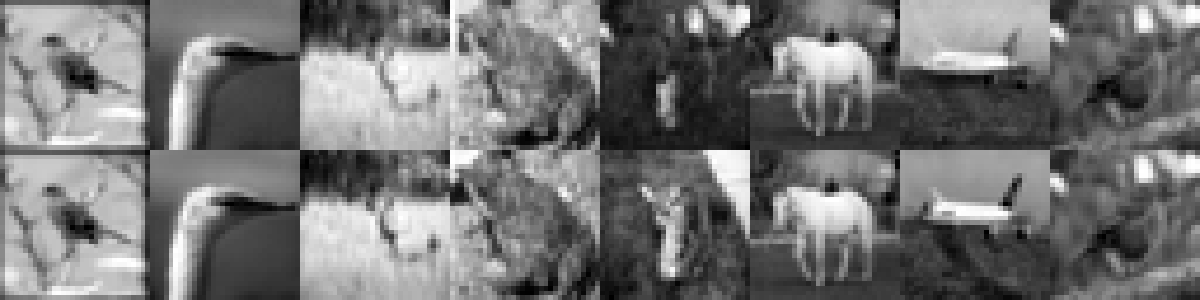}
        \caption{VPSDE. $\text{Fmem: }64.856\%, \text{FID: }138.768$}
        \label{fig:memorization-hold_cifar}
    \end{subfigure}

    \begin{subfigure}[b]{\textwidth}
        \centering
        \includegraphics[width=0.6\textwidth]{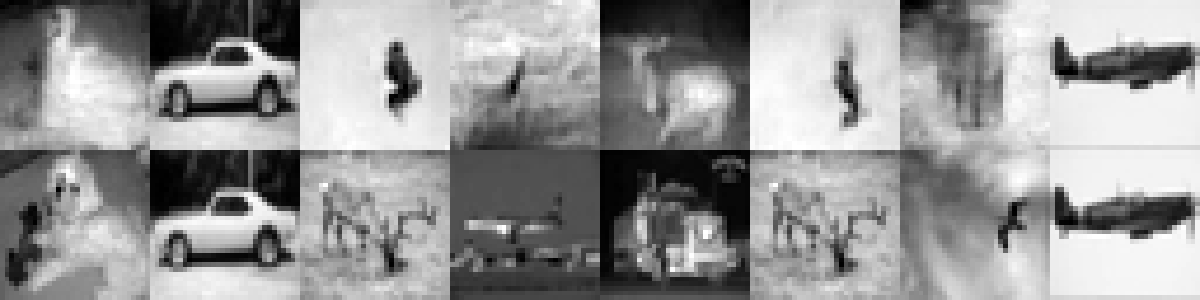}
        \caption{HOLD $n=2$. $\text{Fmem: } 14.410\%,\text{FID: }131.407$}
        \label{fig:memorization-vpsde_cifar}
    \end{subfigure}

    \begin{subfigure}[b]{\textwidth}
        \centering
        \includegraphics[width=0.6\textwidth]{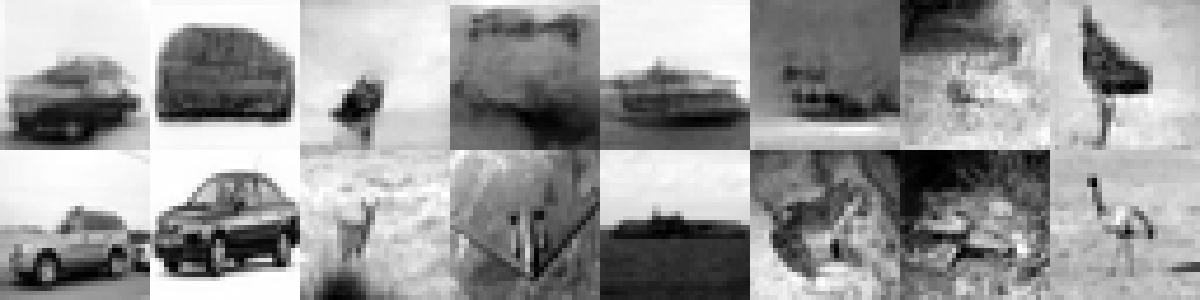}
        \caption{HOLD $n=3$. $\text{Fmem: } 0.598\%,\text{FID: }147.492$}
        \label{fig:memorization-third_cifar}
    \end{subfigure}

    \caption{Nearest training neighbors for different models at $1{,}000{,}000$ training iterations with $1024$ training images on the CIFAR-10 dataset. Each first row contains the generated images, and each second row contains the corresponding nearest neighbors.}
    \label{fig:cifar10memorization-three-models}
\end{figure}

\subsection{Training Losses}
This section contains Figure \ref{fig:traininglosses}, that presents the training losses for $n_{\text{train}} = 256$ for the VPSDE and HOLD $n=2,3$. The slightly higher training losses of HOLD $n=2$ and $n=3$ respectively may be explained by the difficulty of learning the optimal empirical score and statistical regularization that is proposed in this work.

\begin{figure}
    \centering
    \includegraphics[width=0.6\linewidth]{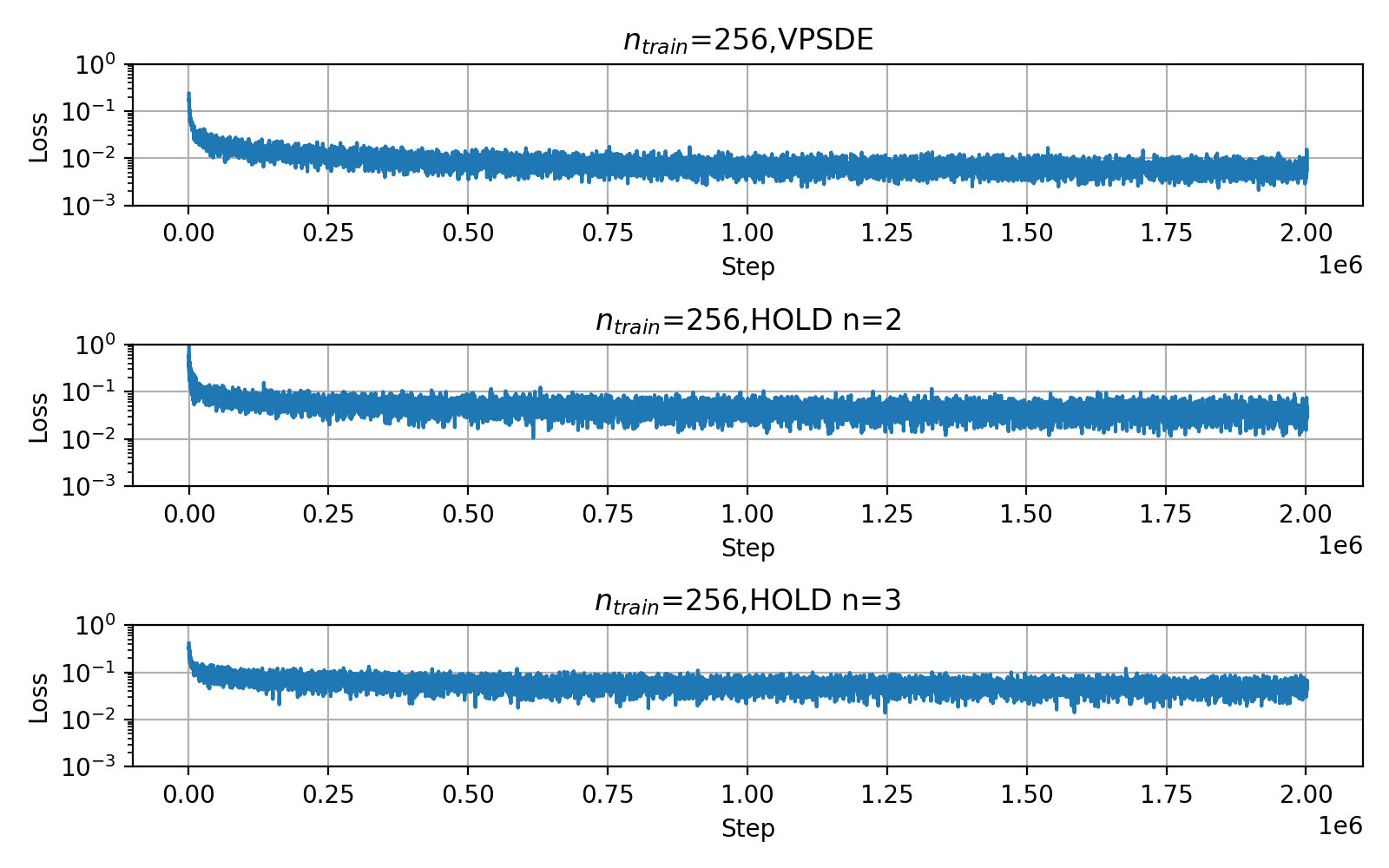}
    \caption{Training losses for $n_{\text{train}} = 256$ for the VPSDE and HOLD $n=2,3$ SDEs.}
    \label{fig:traininglosses}
\end{figure}

\section{Additional CelebA images with FIDs and Fmem}
\label{app:visual_results}

This appendix section presents more generated samples from the CelebA experiment for different number of training images $n_{\text{train}}$.

\begin{figure}[ht]
    \centering

    \begin{subfigure}{0.45\textwidth}
        \centering
        \includegraphics[width=\textwidth]{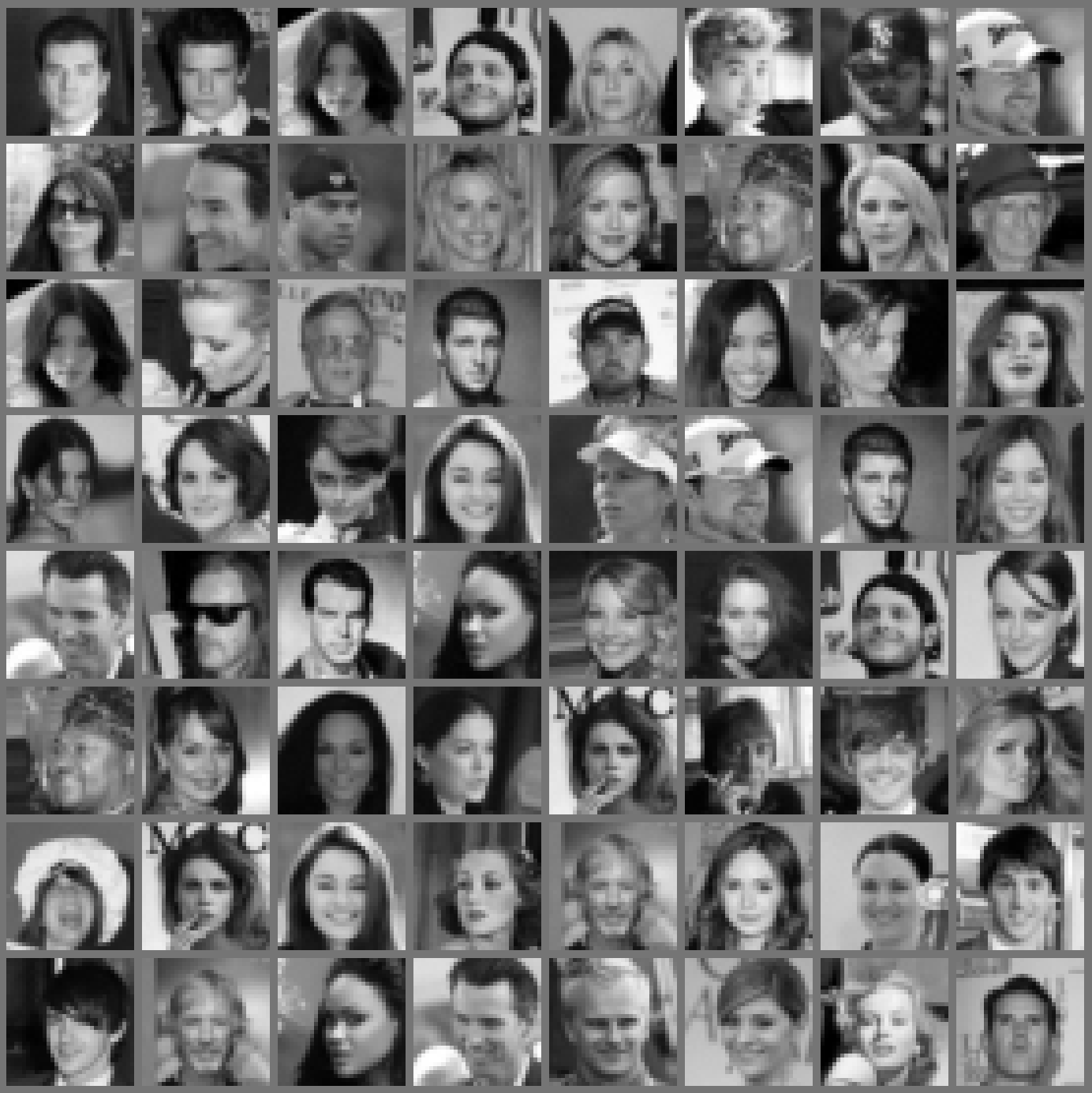}
        \caption{VPSDE. $\text{Fmem: }99.807\%, \text{FID: }58.405$}
    \end{subfigure}
    \hfill
    \begin{subfigure}{0.45\textwidth}
        \centering
        \includegraphics[width=\textwidth]{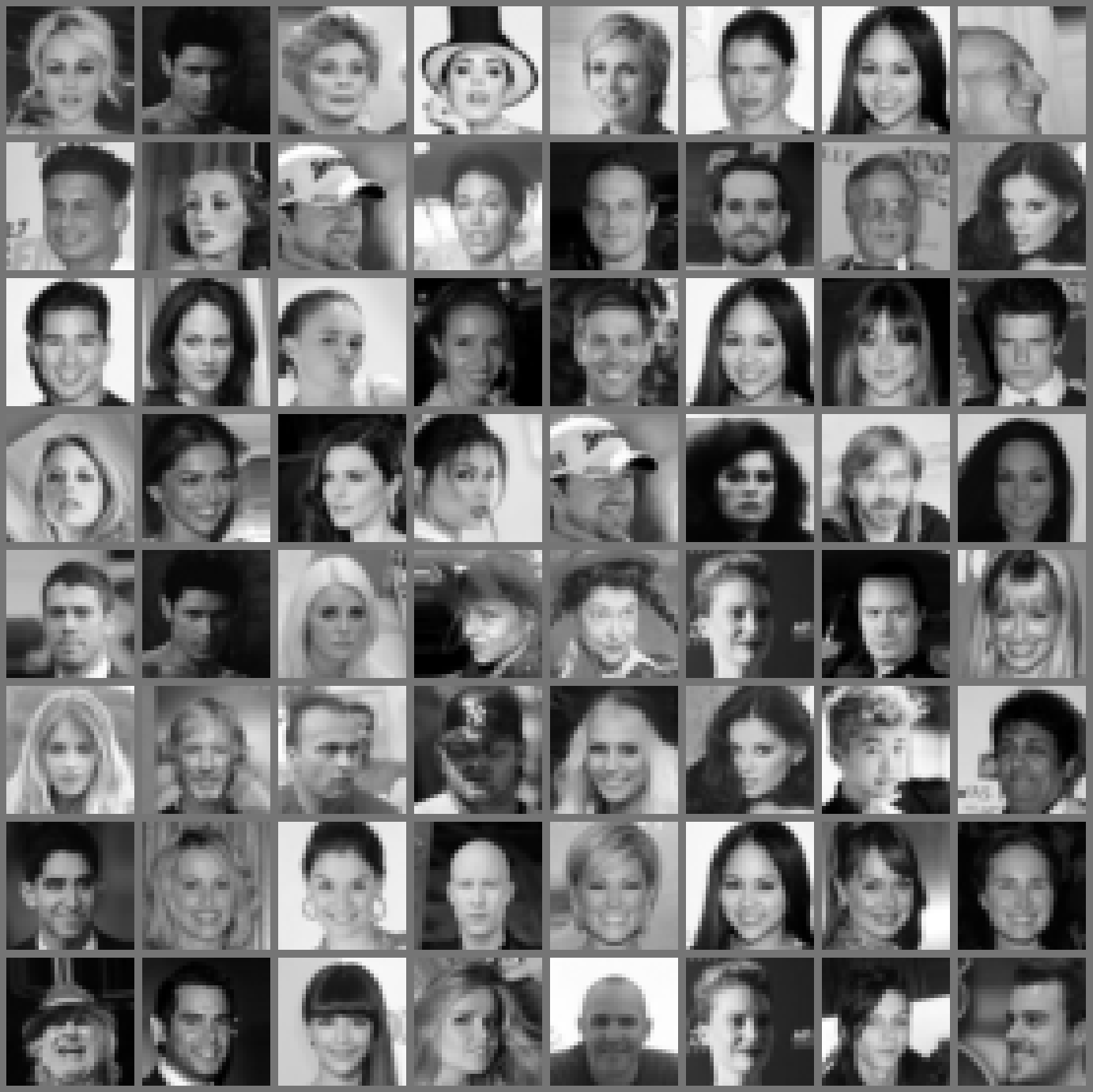}
        \caption{HOLD $n=2$. $\text{Fmem: } 99.522\%,\text{FID: }58.421$}
    \end{subfigure}

    \vspace{0.5em}

    \begin{subfigure}{0.45\textwidth}
        \centering
        \includegraphics[width=\textwidth]{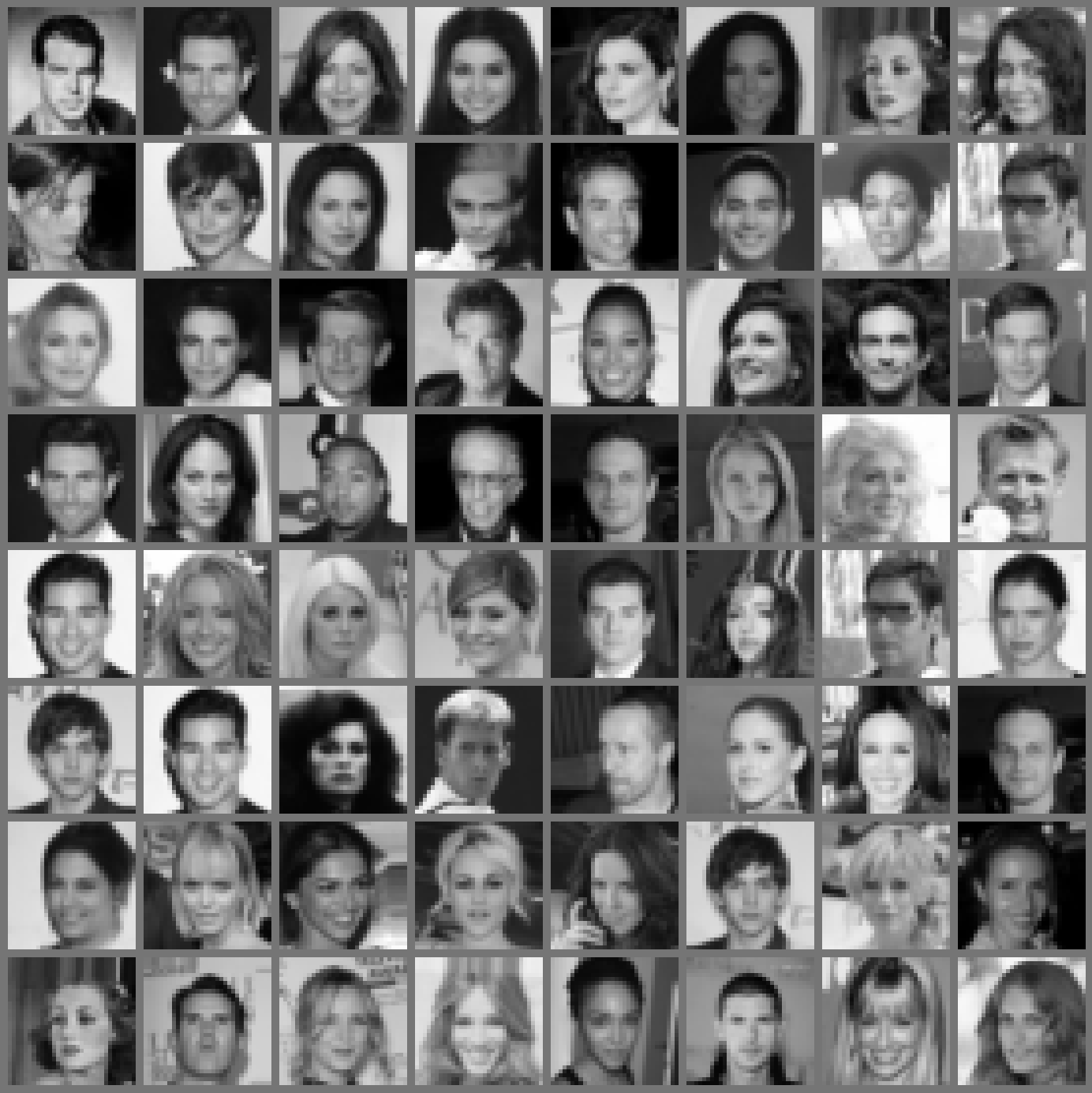}
        \caption{HOLD $n=3$. $\text{Fmem: }90.148\%, \text{FID: }53.066$}
    \end{subfigure}
    \hfill
    \begin{subfigure}{0.45\textwidth}
        \centering
        \includegraphics[width=\textwidth]{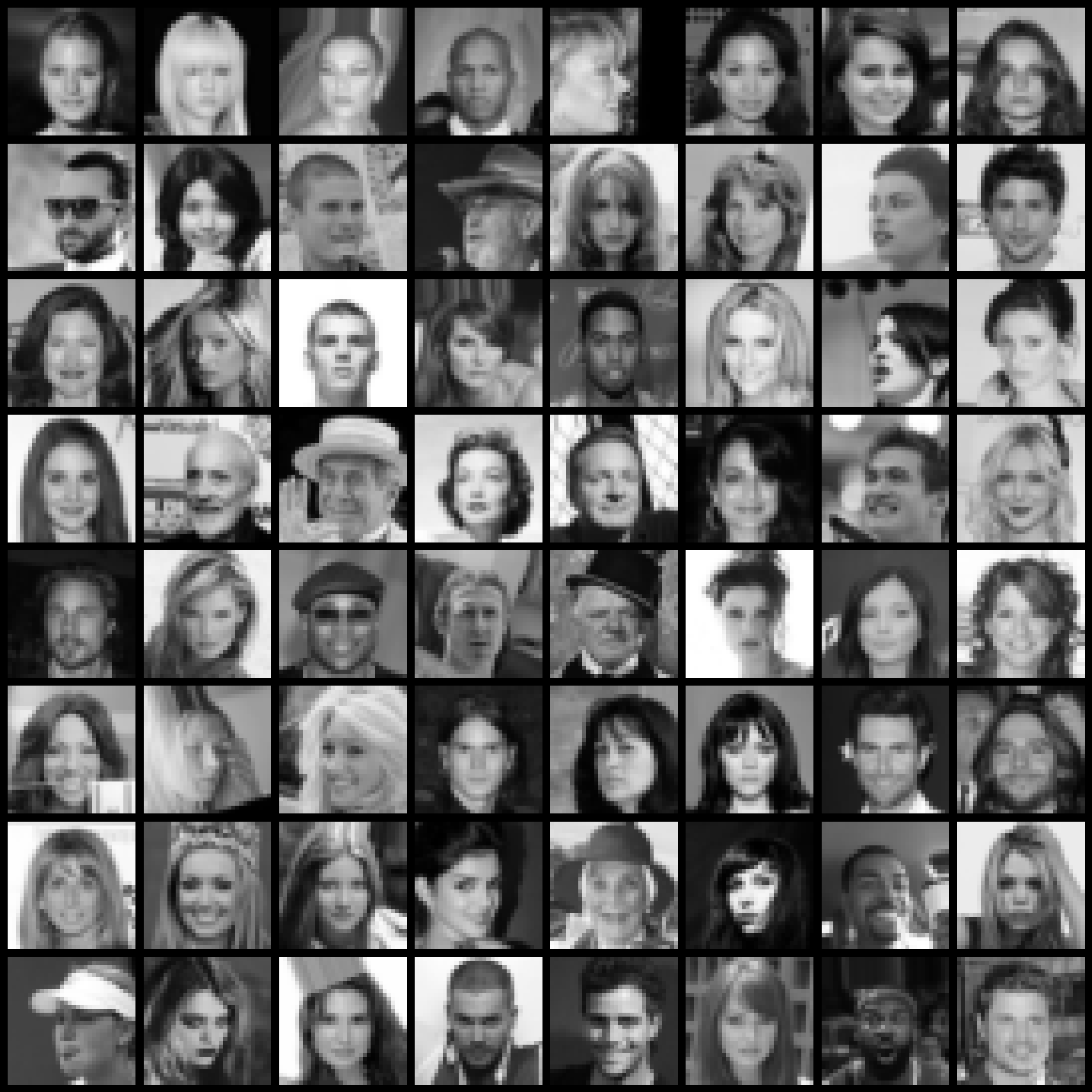}
        \caption{Samples from Training dataset}
    \end{subfigure}

    \caption{Celeba comparison for $n_{\text{train}} = 256$ training samples at $1{,}000{,}000$ training iterations.}
    \label{fig:celebasamples256}
\end{figure}

\begin{figure}[ht]
    \centering
    \begin{subfigure}{0.45\textwidth}
        \centering
        \includegraphics[width=\textwidth]{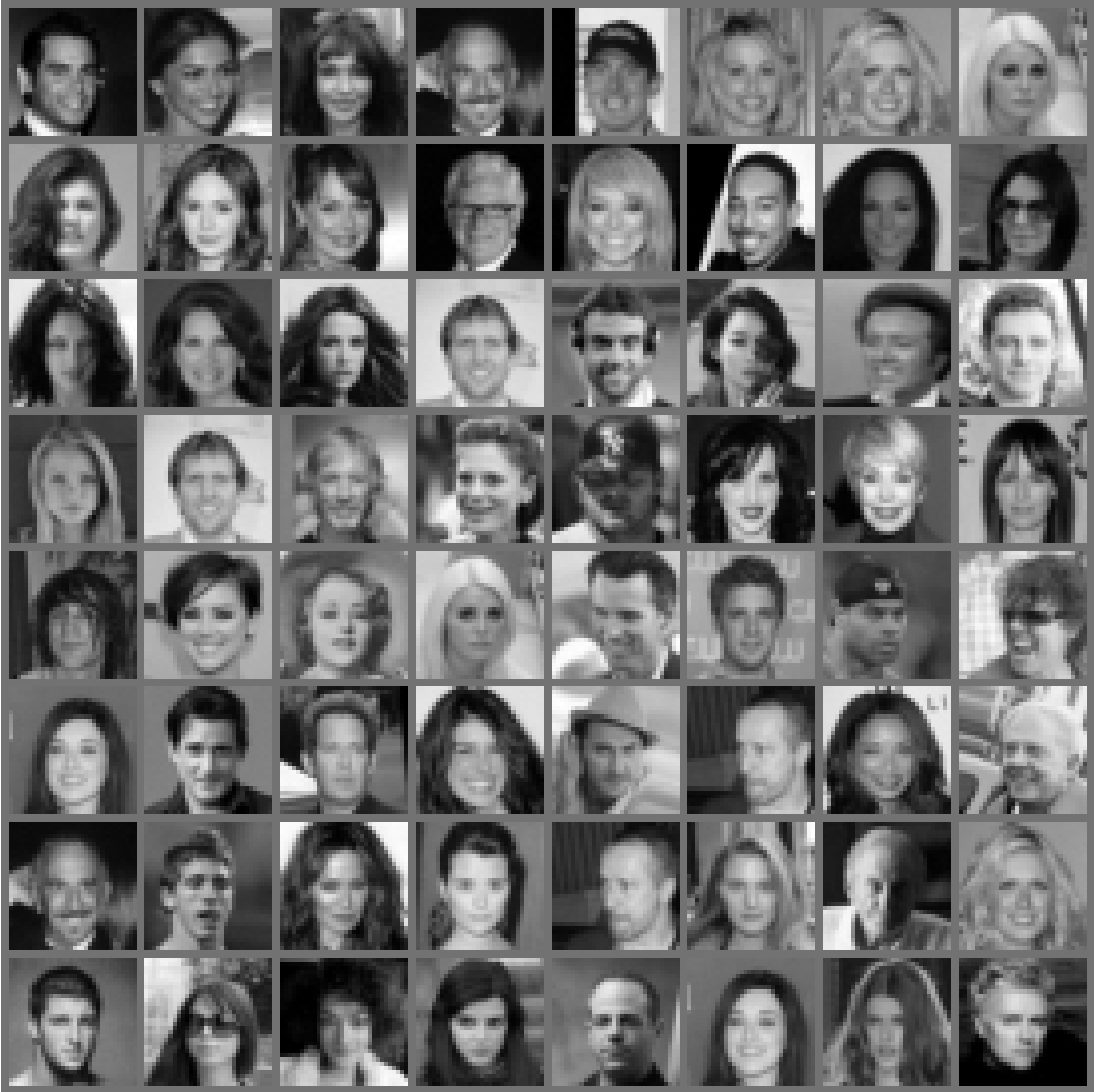}
        \caption{VPSDE. $\text{Fmem: }98.724\%, \text{FID: }51.945$}
    \end{subfigure}
    \hfill
    \begin{subfigure}{0.45\textwidth}
        \centering
        \includegraphics[width=\textwidth]{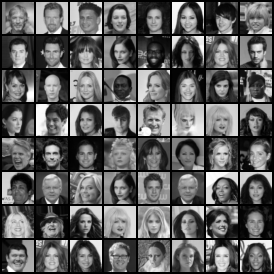}
        \caption{HOLD $n=2$. $\text{Fmem: } 89.363\%,\text{FID: }46.601$}
    \end{subfigure}

    \vspace{0.5em}

    \begin{subfigure}{0.45\textwidth}
        \centering
        \includegraphics[width=\textwidth]{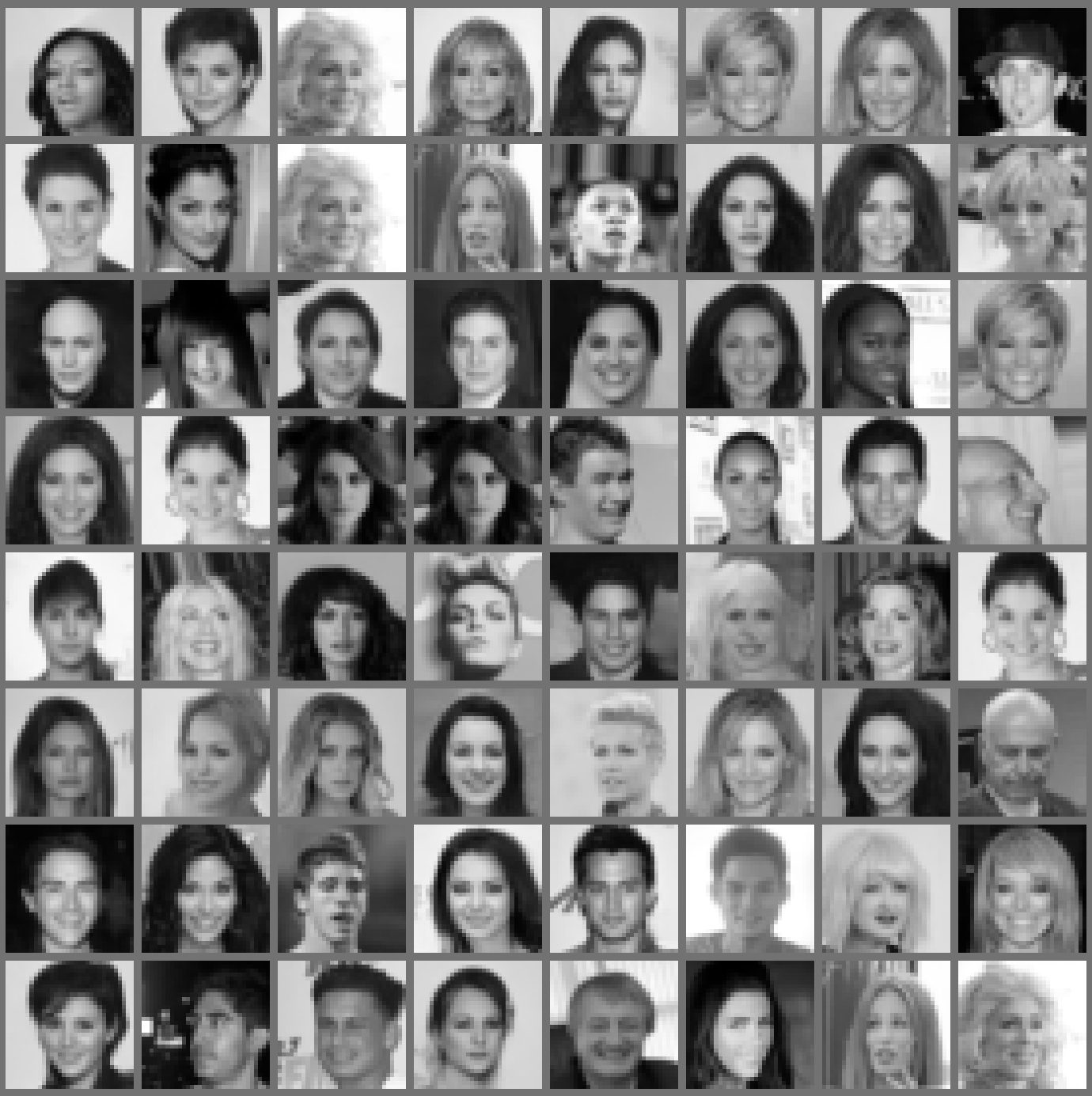}
        \caption{HOLD $n=3$. $\text{Fmem: }50.313\%, \text{FID: }47.817$}
    \end{subfigure}
    \hfill
    \begin{subfigure}{0.45\textwidth}
        \centering
        \includegraphics[width=\textwidth]{celebapics/celebaTraining.png}
        \caption{Samples from Training dataset}
    \end{subfigure}

    \caption{Celeba comparison for $n_{\text{train}} = 512$ training samples at $1{,}000{,}000$ training iterations.}
    \label{fig:celebasamples512}
\end{figure}

\begin{figure}[ht]
    \centering
    \begin{subfigure}{0.45\textwidth}
        \centering
        \includegraphics[width=\textwidth]{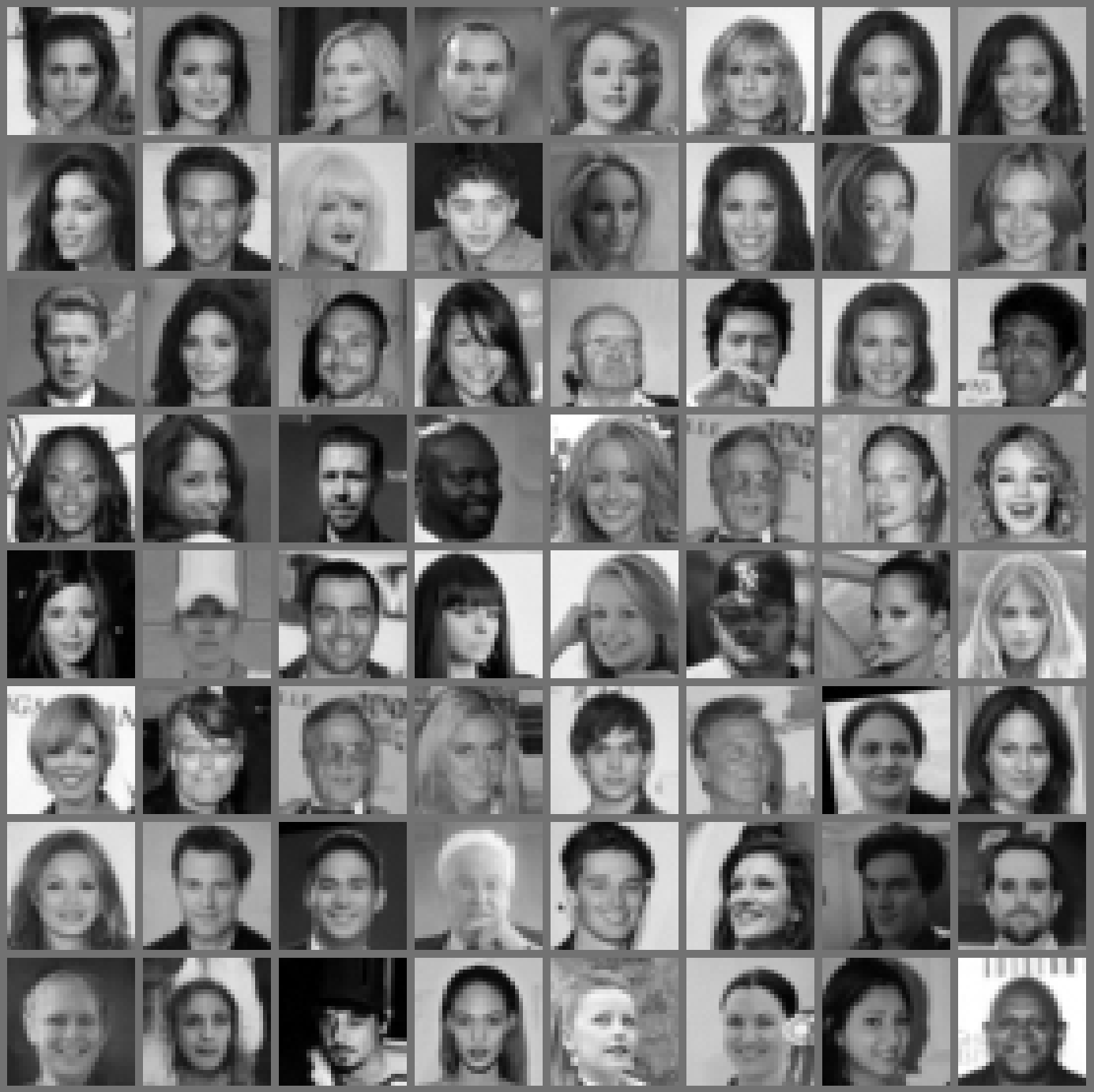}
        \caption{VPSDE. $\text{Fmem: }77.915\%, \text{FID: }47.012$}
    \end{subfigure}
    \hfill
    \begin{subfigure}{0.45\textwidth}
        \centering
        \includegraphics[width=\textwidth]{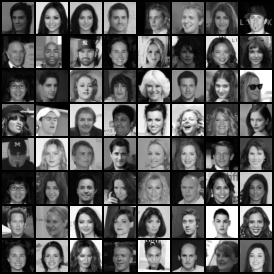}
        \caption{HOLD $n=2$. $\text{Fmem: } 47.024\%,\text{FID: }42.328$}
    \end{subfigure}

    \vspace{0.5em}

    \begin{subfigure}{0.45\textwidth}
        \centering
        \includegraphics[width=\textwidth]{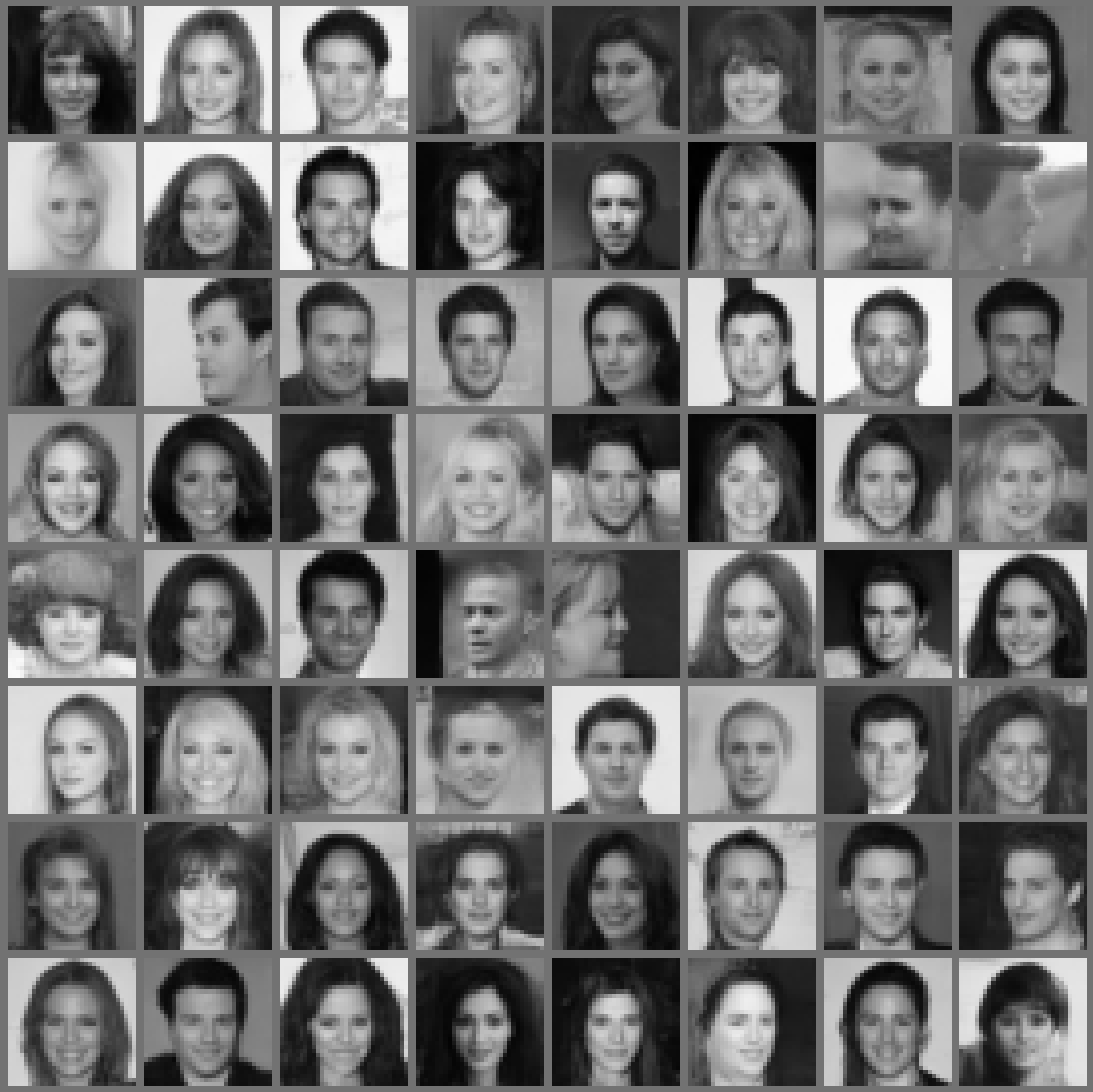}
        \caption{HOLD $n=3$. $\text{Fmem: }5.962\%, \text{FID: }60.380$}
    \end{subfigure}
    \hfill
    \begin{subfigure}{0.45\textwidth}
        \centering
        \includegraphics[width=\textwidth]{celebapics/celebaTraining.png}
        \caption{Samples from Training dataset}
    \end{subfigure}

    \caption{Celeba comparison for $n_{\text{train}} = 1024$ training samples at $1{,}000{,}000$ training iterations.}
    \label{fig:celebasamples1024}
\end{figure}

\begin{figure}[ht]
    \centering

    \begin{subfigure}{0.45\textwidth}
        \centering
        \includegraphics[width=\textwidth]{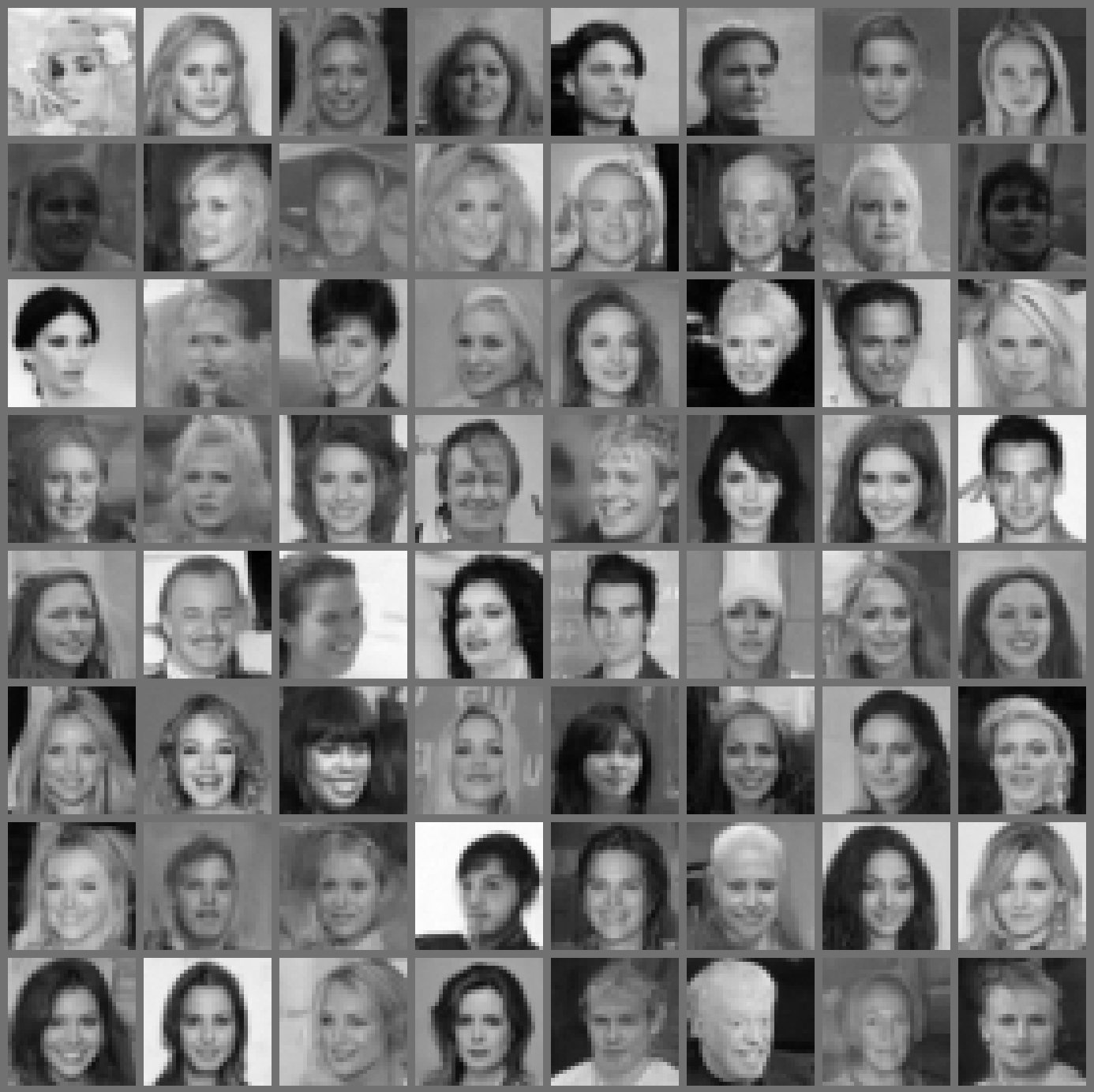}
        \caption{VPSDE. $\text{Fmem: }18.579\%, \text{FID: }49.168$}
    \end{subfigure}
    \hfill
    \begin{subfigure}{0.45\textwidth}
        \centering
        \includegraphics[width=\textwidth]{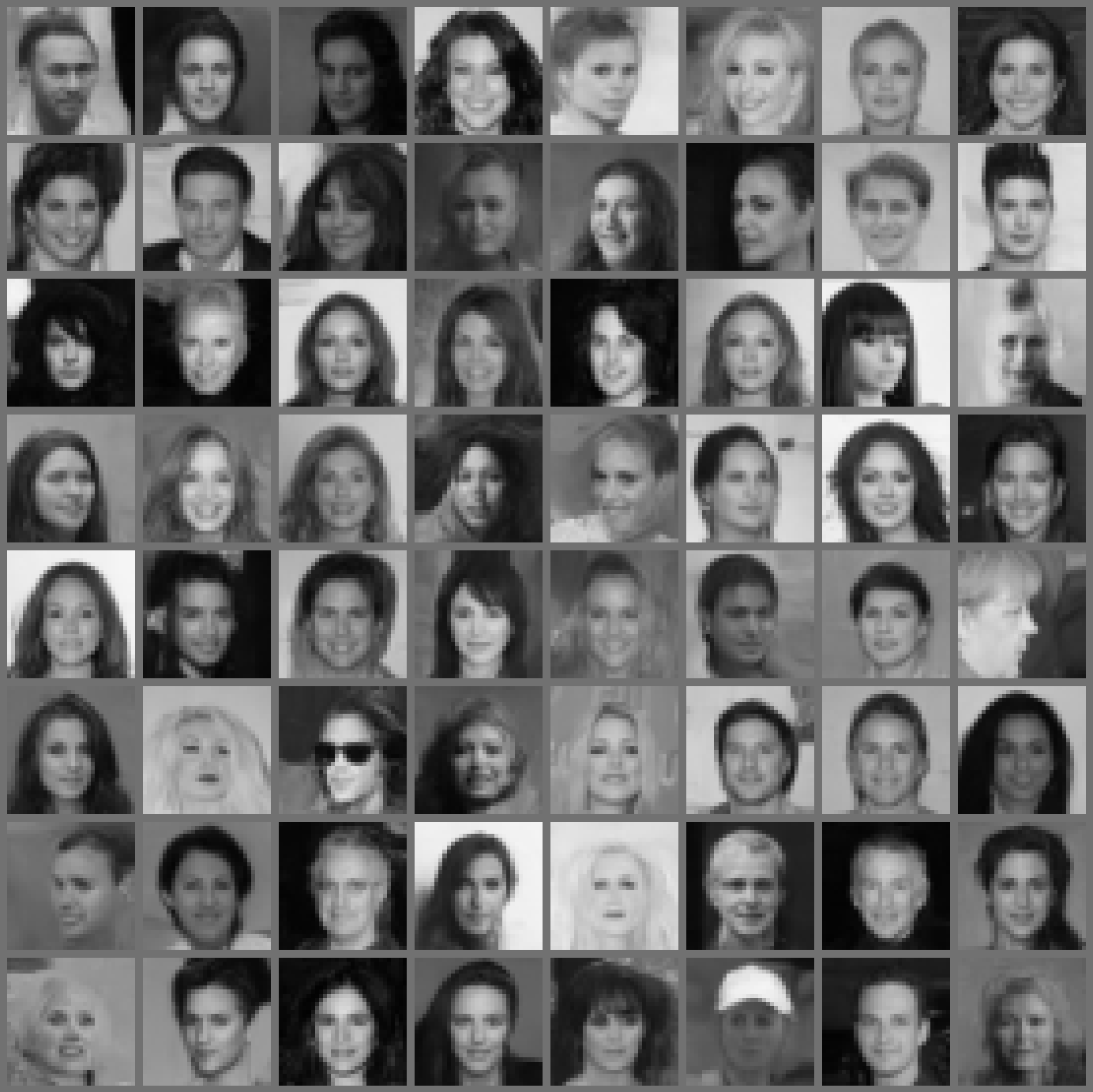}
        \caption{HOLD $n=2$. $\text{Fmem: } 2.845\%,\text{FID: }54.236$}
    \end{subfigure}

    \vspace{0.5em}

    \begin{subfigure}{0.45\textwidth}
        \centering
        \includegraphics[width=\textwidth]{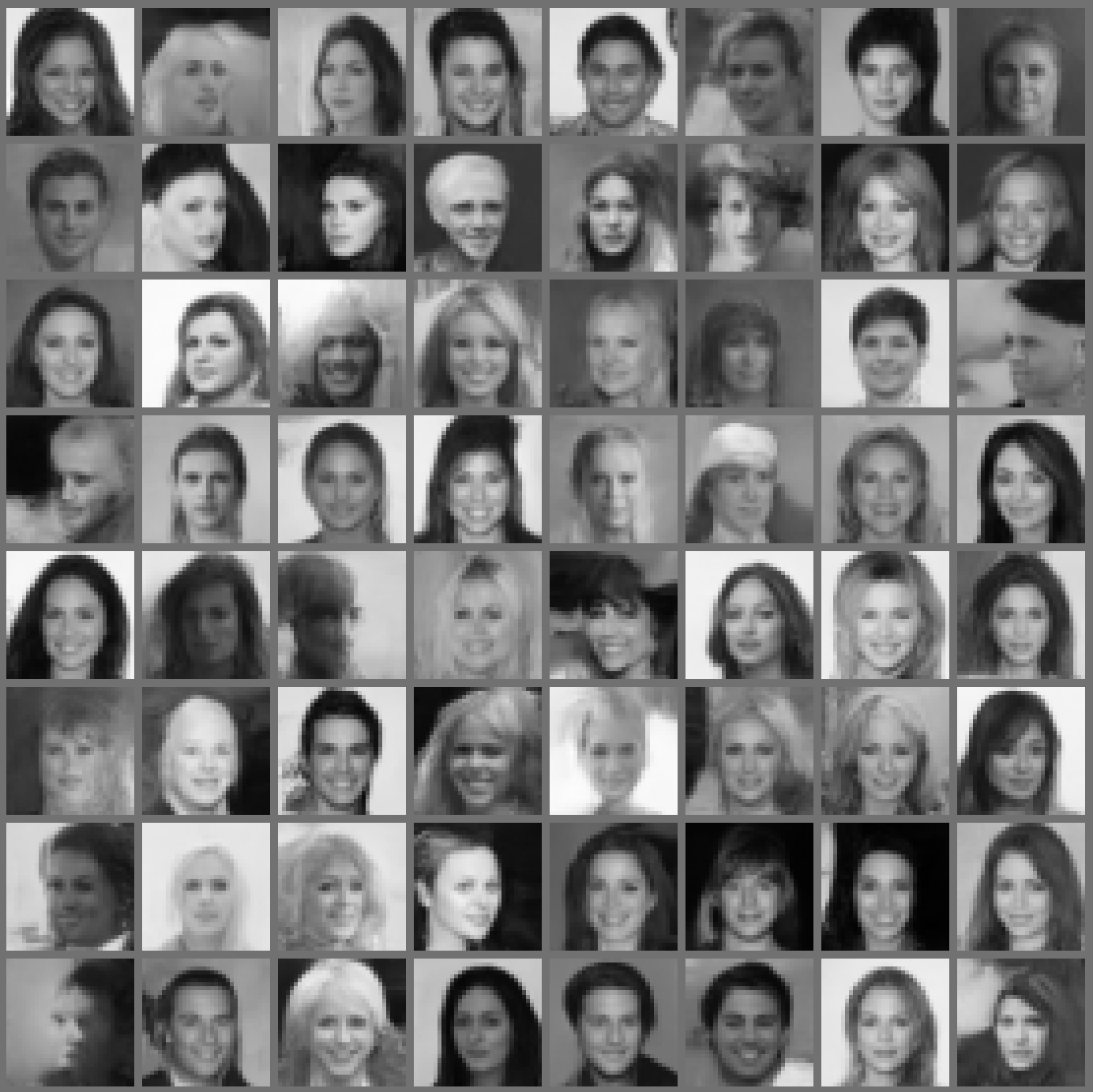}
        \caption{HOLD $n=3$. $\text{Fmem: }0.000\%, \text{FID: }53.560$}
    \end{subfigure}
    \hfill
    \begin{subfigure}{0.45\textwidth}
        \centering
        \includegraphics[width=\textwidth]{celebapics/celebaTraining.png}
        \caption{Samples from Training dataset}
    \end{subfigure}

    \caption{Celeba comparison for $n_{\text{train}} = 2048$ training samples at $1{,}000{,}000$ training iterations.}
    \label{fig:celebasamples2048}
\end{figure}

\section{Miscellaneous details}
\label{app:miscExpDetails}
For all experiments, a learning rate of $1 \times 10^{-4}$, the Adam optimizer, 32 base channels, a dropout rate of 0.1, and 1000 diffusion model steps were used. For the VPSDE, a linear noise schedule with $\beta_0=1\times10^{-3}, \beta_1=10.0$ was used, and the HOLD runs used $L^{-1} = 1.0$. The UNet (the same architecture as in \citep{bonnaire2025why}) uses channel multipliers (1, 2, 4) with self-attention applied at the middle two resolution levels. Experiments were conducted on a single NVIDIA H200 SXM5 GPU, with 141 GB of VRAM. For each experiment, training and generation combined took approximately 3 days of wall clock time, and separate FID and Fmem calculations took negligible amounts of time.


\end{document}